\documentclass{article} 
\usepackage{iclr2027_conference,times}
\usepackage{booktabs}
\usepackage{booktabs,xcolor,pifont}
\newcommand{\cmark}{\textcolor{green!55!black}{\ding{51}}}
\newcommand{\xmark}{\textcolor{red!75!black}{\ding{55}}}
\usepackage{enumitem}
\usepackage{multirow}

\usepackage{amsmath,amsfonts,bm}

\def\eqref#1{equation~\ref{#1}}

\def\1{\bm{1}}

\DeclareMathAlphabet{\mathsfit}{\encodingdefault}{\sfdefault}{m}{sl}
\SetMathAlphabet{\mathsfit}{bold}{\encodingdefault}{\sfdefault}{bx}{n}

\usepackage{hyperref}
\usepackage{url}
\usepackage[table]{xcolor}
\usepackage{graphicx}
\usepackage{float}      
\usepackage{caption}
\usepackage{amsthm}
\newtheorem{theorem}{Theorem}
\newtheorem{proposition}{Proposition}
\usepackage{pifont}

\usepackage[most]{tcolorbox}

\definecolor{suppbg}{RGB}{245,247,249}
\definecolor{suppborder}{RGB}{225,229,234}

\title{How Much Planning Is Enough?\\
Reducing Search and Computation in World-Model Planning}

\author{
\normalfont
Changbai Li$^{1*}$ \quad
Sirui Li$^{2*}$ \quad
Yichen Yang$^{1}$ \quad
Tongfei Chen$^{1}$ \\[0.3em]
Zichao Feng$^{1}$ \quad
Shuwei Shao$^{3}$ \quad
Huobin Tan$^{1\dagger}$ \\[0.6em]
$^{1}$Beihang University \qquad
$^{2}$Chengdu University of Technology \qquad
$^{3}$Nanyang Technological University
}

\iclrfinalcopy 
\begin{document}

\maketitle

\begingroup
\renewcommand{\thefootnote}{\fnsymbol{footnote}}
\footnotetext[1]{Equal contribution.}
\footnotetext[2]{Corresponding author.}
\endgroup

\begin {abstract}
Visual world models enable goal-directed control through decision-time
action search, but their deployment efficiency is often limited by
conservatively large planning budgets.
We show that competitive task performance can be achieved without
agreement with the Full-budget action, that sufficient budgets vary
across model--task pairs, and that iterative planners repeatedly encode
solve-invariant context.
To address these inefficiencies, we propose {SufficientPlan}, a
simple deployment framework that requires no modification to pretrained
world models or planner updates.
Its {Paired Sequential Budget Certification (PSBC)} component uses
paired closed-loop evidence to search for and certify a reduced model--task-specific
budget within a predefined Full-performance tolerance.
Its {Static-Context Reuse (SCR)} component caches observation and
goal representations across search iterations while preserving
candidate-dependent planning and selected actions.
Experiments across multiple world-model backbones and visual-control tasks
show that SufficientPlan substantially reduces search budgets and planning
latency while maintaining competitive control performance.
\end {abstract}

\begin{figure*}[!h]
\centering
\includegraphics[width=1.0\textwidth]{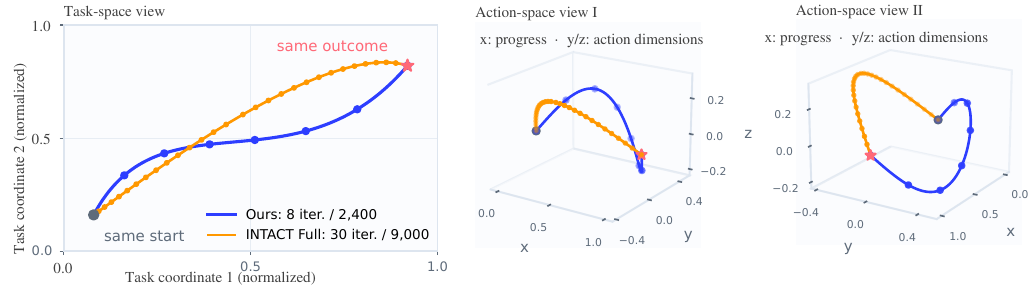}
\caption{
\textbf{Task success without Full-budget action agreement.}
On INTACT Push-T, the calibrated and Full planners follow different
task- and action-space trajectories but both achieve 96\% success, using
2,400 and 9,000 scored sequences per solve, respectively.
}
\label{fig:action_saturation}
\end{figure*}

\section{Introduction}
\vspace{-0.5em}
\label{sec:introduction}

Latent visual world models, including those based on joint-embedding
predictive architectures (JEPAs), predict the effects of candidate actions
in a learned representation space and support goal-directed control
through model-based planning~\citep{dino_wm,lewm,causal_jepa,intact}.
In this paradigm, decision-time search is critical for translating learned
dynamics into effective control, but its computational cost directly
limits deployment efficiency~\citep{cem,mppi,icem}.
Existing world-model planners commonly use a large, globally fixed search
budget for different models and tasks~\citep{pets,dino_wm,lewm,intact}.
Such budgets determine how many candidate action sequences are evaluated
and how extensively they are refined before execution.
However, a fixed computational allowance does not account for whether
additional search meaningfully improves closed-loop task performance. However, we find that this practice introduces two complementary forms of redundancy. First, additional action refinement may continue even when a reduced
budget and Full planning achieve the same aggregate task performance. As shown in
Fig.~\ref{fig:action_saturation}, the calibrated and Full planners follow distinct task- and action-space trajectories but achieve the same 96\% success rate using 2,400 and 9,000 scored sequences, respectively. Moreover, the budget required to reach this outcome-sufficient regime varies considerably across model--task pairs (See Tab.~\ref{tab:planning_budget_success}). Second, iterative planners repeatedly encode the same observation history
and goal during candidate refinement, even though these inputs remain unchanged within a planning solve. As shown in Fig.~\ref{fig:scr_mechanism}, the original planner places observation-history and goal encoding inside the CEM loop, causing the same solve-static inputs to be re-encoded in every iteration and resulting in $2K$ static-encoder calls over $K$ refinement iterations.

To address these problems, we propose SufficientPlan, a simple, effective, and plug-and-play framework for improving the success-compute trade-off of world-model planning without retraining the underlying model or changing its control objective. SufficientPlan contains two complementary components.
First, Paired Sequential Budget Certification (PSBC) evaluates search budget proposals on paired development episodes and certifies a lower model--task-specific budget that remains within a predefined performance tolerance relative to Full planning. As shown in Fig.~\ref{fig:psbc}, PSBC compares each current budget with the Full reference on the same episodes: uncertain cases receive additional paired evidence, insufficient proposals trigger a larger budget, and the first proposal satisfying the certification criterion is frozen for held-out deployment. Unlike action-matching criteria, PSBC directly evaluates task outcomes,
allowing a reduced budget to be certified even when its selected actions
do not agree with those of the Full-budget planner. Second, Static-Context Reuse (SCR) computes the solve-invariant observation and goal representations once and reuses them across all search iterations. As shown in Fig.~\ref{fig:scr_mechanism}, SCR moves these solve-static encoding operations outside the CEM loop, reducing repeated solve-static encoding operations from $2K$ to $2$ over $K$ iterations. SCR preserves candidate-dependent action encoding, latent rollout, cost evaluation, and optimizer updates, thereby reducing repeated computation without changing the planner's selected actions or closed-loop behavior.

Together, PSBC and SCR address two orthogonal sources of planning inefficiency: PSBC removes unnecessary search, while SCR removes redundant computation within the remaining search. Because both components operate entirely at deployment time, SufficientPlan can be applied to existing sampling-based world-model planners as a lightweight plug-in. Experiments across multiple world-model backbones and control benchmarks show that SufficientPlan substantially reduces candidate evaluation and planning latency while preserving competitive closed-loop task performance.
\\Our main contributions are summarized below:
\vspace{-0.5em}

\begin{itemize}[itemsep=-0.3pt, topsep=0pt, leftmargin=15pt]
\item We identify a mismatch between closed-loop task success and
Full-budget action agreement: comparable task performance can be achieved
with different actions, and the sufficient search budget varies
substantially across models and tasks.
\item We introduce SufficientPlan, a simple and plug-and-play planning-efficiency framework combining PSBC, which sequentially searches for and certifies a reduced outcome-sufficient search budget, and SCR, which exactly reuses solve-invariant representations during iterative planning.
\item We evaluate SufficientPlan across multiple world-model backbones and control benchmarks, demonstrating substantial reductions in search computation and planning latency while preserving closed-loop performance.

\end{itemize}

\begin{figure*}[t]
\centering
\includegraphics[width=1.0\textwidth]{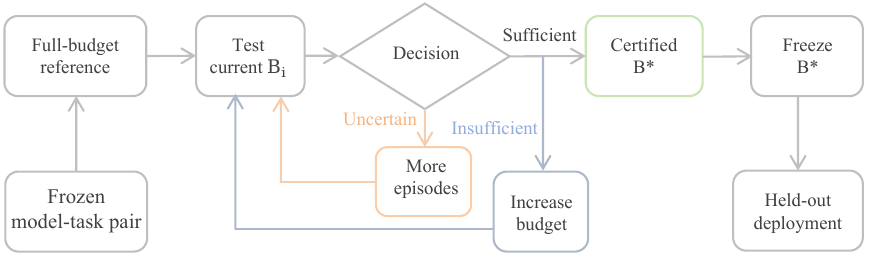}
\vspace{-0.9em}
\caption{
\textbf{Paired Sequential Budget Certification (PSBC).}
PSBC compares the current budget with the full-budget reference on paired calibration episodes. Uncertain cases receive additional paired evidence, insufficient budgets are increased, the first proposal satisfying the certification criterion is returned as
$B^\star$ and frozen for held-out deployment.
}
\label{fig:psbc}
\end{figure*}

\section {Related Work}
\label {sec:related_work}
\vspace {-0.9em}

\noindent\textbf {Latent World Models for Visual Planning.}
World models learn predictive representations of environment dynamics
for decision making~\citep {world_models}.
PlaNet and the Dreamer family perform planning or policy optimization in
compact latent spaces~\citep {planet,dreamer,dreamerv2,dreamerv3}, while
MuZero, EfficientZero, and TD-MPC show that task-relevant latent dynamics
can support control without pixel-perfect future reconstruction
~\citep {muzero,efficientzero,td_mpc2}.
Joint-embedding predictive architectures, including I-JEPA, V-JEPA, and
V-JEPA2, further learn predictive representations without reconstructing
raw pixels~\citep {ijepa,vjepa,vjepa2}.
DINO-WM uses pretrained visual features for goal-conditioned planning
~\citep {dino_wm}; LeWM learns a lightweight world model end-to-end from
pixels, and Fast-LeWM accelerates candidate evaluation through
action-prefix prediction~\citep {lewm,fast_lewm}.
INTACT learns an intent-to-action mapping for direct control and optional
local search~\citep {intact}.
Object-centric methods instead represent scenes as entities and model
their dynamics and interactions
~\citep {slot_attention,savi,slotformer,sold,causal_jepa}.
Unlike these methods, which primarily improve representations, dynamics,
or action priors, SufficientPlan keeps the pretrained model and planner
fixed and reduces decision-time computation.

\noindent\textbf {Efficient Decision-Time Planning.}
Sampling-based optimizers such as CEM and MPPI optimize action sequences
under learned dynamics without requiring differentiable models
~\citep {cem,mppi}.
Methods such as PETS and PDDM apply this principle in receding-horizon
control~\citep {pets,pddm}, but their cost grows with the candidate
population and search depth.
Prior work improves efficiency using learned policies or behavior priors
~\citep {poplin,mbop} or
sample-efficient CEM refinements such as elite reuse and distribution
warm-starting~\citep {icem}.
Adaptive replanning instead reduces how frequently the planner is invoked
across environment steps~\citep {adarep}.
SufficientPlan targets two complementary sources of deployment
overcomputation.
\textbf {Paired Sequential Budget Certification (PSBC)} uses paired
closed-loop evidence to select a model--task-specific search budget,
which is calibrated on a development set and frozen for held-out
deployment.
Unlike adaptive replanning or per-state early stopping, PSBC changes the
budget assigned to each model--task pair without modifying the world
model or planner update rule.
\textbf {Static-Context Reuse (SCR)} encodes the solve-invariant
observation history and goal once and reuses them across refinement
iterations, while preserving candidate evaluation, planner updates, and
the selected action.
Thus, PSBC reduces the amount of deployed search, whereas SCR removes
repeated computation within the retained search.

\begin{figure*}[t]
\centering
\includegraphics[width=1.0\textwidth]{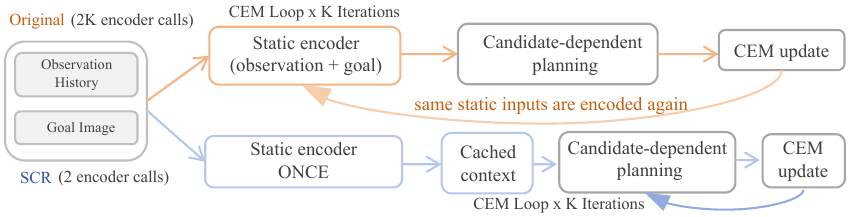}
\vspace{-0.9em}
\caption{
\textbf{Static-Context Reuse (SCR).}
SCR encodes the solve-invariant observation history and goal once and
reuses them across $K$ CEM iterations, reducing cost-stage static-encoder
calls from $2K$ to $2$ without changing candidate-dependent planning or
CEM updates.
}
\label{fig:scr_mechanism}
\end{figure*}

\section{Motivation and Analysis}
\label{sec:motivation}
\vspace{-0.5em}

\noindent\textbf{Outcome Sufficiency without Full-Action Agreement.}
Iterative world-model planners refine candidate actions using predicted
latent costs~\citep{cem,pets,dino_wm,lewm}, but successful control need
not reproduce the Full-budget action.
On INTACT Push-T, the calibrated planner uses 8 CEM iterations and
2,400 scored sequences per solve, compared with 30 iterations and
9,000 sequences for Full planning; both achieve 96\% success on the
development set.
Fig.~\ref{fig:action_saturation} illustrates distinct trajectories
reaching the same goal, while Fig.~\ref{fig:action_vs_success} quantifies
action differences despite equal aggregate success.
Thus, requiring action agreement with Full planning could prolong
search after the target level of closed-loop performance has been reached.
We instead assess budget sufficiency using paired task outcomes.


\noindent\textbf{Model--Task Variation in Certified Budgets.}
The budgets certified by PSBC vary substantially across model--task pairs.
As shown in Tab.~\ref{tab:planning_budget_success}, the selected operating
budgets range from 3.3\% to 80.0\% of Full planning across the evaluated
pairs.
For example, PSBC selects 3.3\% of the Full budget for INTACT on Cube,
but 80.0\% for LeWM on Reacher.
Consequently, a uniformly small budget may underallocate search for
refinement-sensitive pairs, while a uniformly large budget can waste
computation on pairs for which smaller budgets are sufficient
These results motivate model--task-specific budget certification rather
than a universal deployment budget.

\noindent\textbf{Repeated Solve-Static Encoding.}
Reducing the search budget removes unnecessary candidate evaluations, but
repeated computation can remain within the retained iterations.
In the iterative planners examined here, the observation history and goal
remain fixed within a planning solve, whereas only the candidate action
sequences change.
Nevertheless, these solve-static inputs are encoded inside each
candidate-evaluation iteration.
As illustrated in Fig.~\ref{fig:scr_mechanism}, a planner with $K$
iterations consequently performs $2K$ static-encoder calls for only two
distinct inputs.

This repetition produces measurable wall-clock overhead.
On Fast-LeWM/Push-T, the analytical reduction in static-encoder calls
increases from 0\% at one iteration to 96.7\% at 30 iterations.
Repeated timing measurements show no reliable speedup at one iteration,
whereas SCR reaches a $2.45\times$ speedup at 30 iterations
(Fig.~\ref{fig:scr_scaling}).
Matched-budget evaluations across INTACT, Fast-LeWM, and C-JEPA
preserve the corresponding episode-success vectors with zero outcome flips
(Fig.~\ref{fig:scr}).
These measurements show that repeated solve-static encoding contributes
latency without changing the evaluated closed-loop outcomes.

Together, the analysis identifies two complementary sources of deployment
overcomputation:
\[
\underbrace{\text{search beyond outcome sufficiency}}
_{\text{excess candidate evaluation}}
\qquad\text{and}\qquad
\underbrace{\text{repeated solve-static encoding}}
_{\text{redundant per-iteration computation}}.
\]
The first motivates outcome-based, model--task-specific budget
certification, while the second motivates reusing invariant
representations within each planning solve.

\begin{figure*}[t]
\centering
\includegraphics[width=1.0\textwidth]{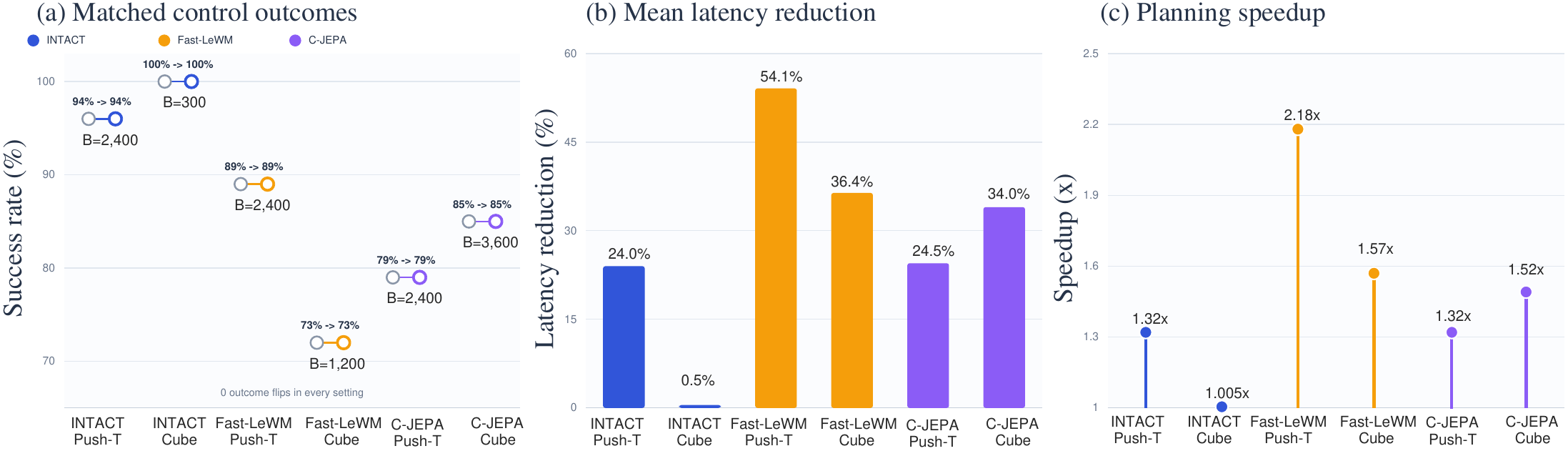}
\vspace{-0.9em}
\caption{
\textbf{Matched-budget evaluation of Static-Context Reuse (SCR).}
(a) Matched-budget success rates and episode-level outcome consistency.
(b) Mean planning-latency reduction.
(c) Planning speedup across backbone--task settings.
SCR preserves all evaluated outcomes while achieving up to
$2.18\times$ speedup.
}
\label{fig:scr}
\end{figure*}

\begin{table*}[t]
\centering
\caption{
Comparison of full-budget planning and PSBC-calibrated planning across different world-model planners and control tasks.
\textbf{Budget} denotes the planning budget relative to the corresponding full planner, and \textbf{Success} denotes the task success rate (\%).
``Full'' indicates the original full planning budget, while ``--'' denotes unavailable results.
}
\label{tab:planning_budget_success}
\small
\setlength{\tabcolsep}{5.5pt}
\renewcommand{\arraystretch}{1.15}
\resizebox{\textwidth}{!}{
\begin{tabular}{ll cc cc cc cc}
\toprule
\multirow{2}{*}{\textbf{Planner}} &
\multirow{2}{*}{\textbf{Deployment}} &
\multicolumn{2}{c}{\textbf{Push-T}} &
\multicolumn{2}{c}{\textbf{Cube}} &
\multicolumn{2}{c}{\textbf{Reacher}} &
\multicolumn{2}{c}{\textbf{TwoRoom}} \\
\cmidrule(lr){3-4}
\cmidrule(lr){5-6}
\cmidrule(lr){7-8}
\cmidrule(lr){9-10}
&
&
\textbf{Budget} & \textbf{Success} &
\textbf{Budget} & \textbf{Success} &
\textbf{Budget} & \textbf{Success} &
\textbf{Budget} & \textbf{Success} \\
\midrule

LeWM
& Full CEM
& Full & 85
& Full & 71
& Full & 82
& Full & 91 \\

LeWM+PSBC
& Calibrated CEM
& 50.0\% & 88
& 13.3\% & 74
& 80.0\% & 82
& 30.0\% & 91 \\

\midrule

Fast-LeWM
& Full CEM
& Full & 95
& Full & 68
& Full & 82
& Full & 94 \\

Fast-LeWM+PSBC
& Calibrated CEM
& 26.7\% & 89
& 13.3\% & 73
& 80.0\% & 81
& 50.0\% & 94 \\

\midrule

PRISM
& Official PoG-MPPI
& Full & 88
& Full & 78
& -- & --
& -- & -- \\

PRISM+PSBC
& Calibrated PoG-MPPI
& 26.7\% & 89
& 6.7\% & 82
& -- & --
& -- & -- \\

\midrule

C-JEPA
& Full CEM
& Full & 82
& Full & 82
& -- & --
& -- & -- \\

C-JEPA+PSBC
& Calibrated CEM
& 26.7\% & 79
& 40.0\% & 85
& -- & --
& -- & -- \\

\midrule

INTACT
& Full Actor-CEM
& Full & 95
& Full & 98
& Full & 85
& Full & 100 \\

INTACT+PSBC
& Calibrated Actor-CEM
& 26.7\% & 94
& 3.3\% & 100
& 3.3\% & 100
& 3.3\% & 100 \\

\midrule

\textbf{SufficientPlan}
& \textbf{Calibrated Actor-CEM + SCR}
& \textbf{26.7\%} & \textbf{94}
& \textbf{3.3\%} & \textbf{100}
& \textbf{3.3\%} & \textbf{100}
& \textbf{3.3\%} & \textbf{100} \\

\bottomrule
\end{tabular}
}
\end{table*}

\begin{figure*}[t]
\centering
\includegraphics[width=1.0\textwidth]{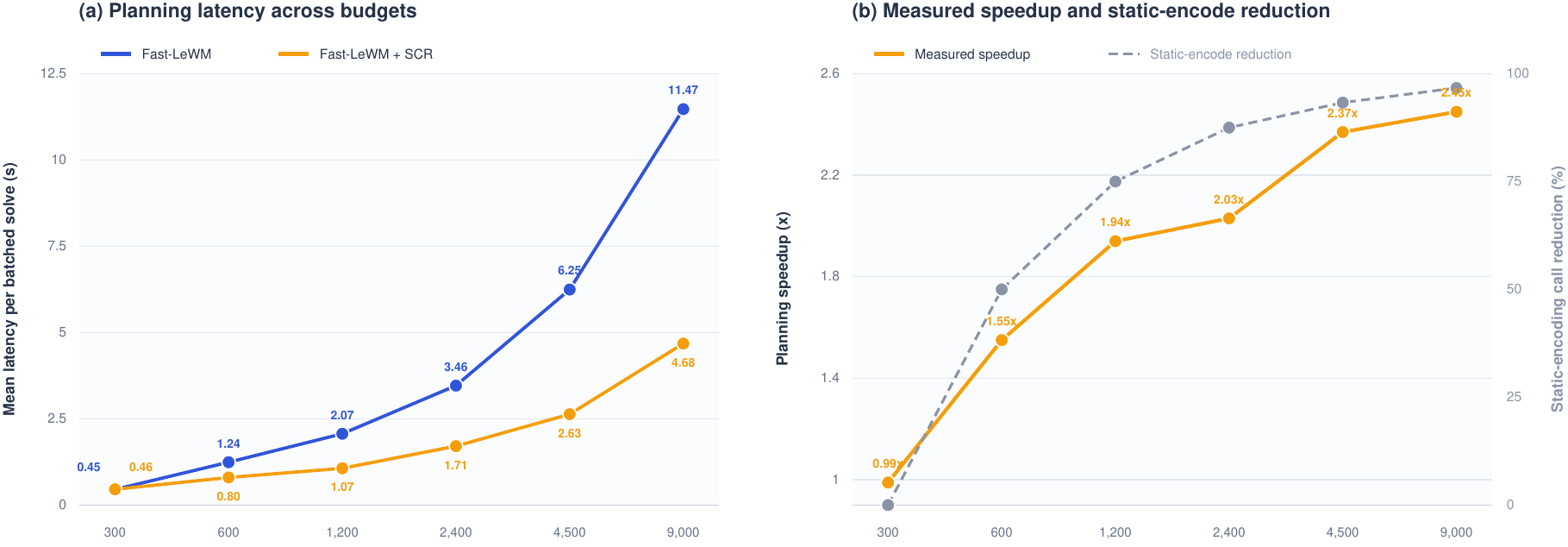}
\vspace{-0.9em}
\caption{
\textbf{SCR scaling with planning depth on Fast-LeWM/Push-T.}
(a) Mean planning latency of the uncached and SCR-enabled planners.
(b) Measured speedup and analytical reduction in solve-static encoder
calls. Repeated measurements show no reliable speedup at one iteration,
whereas SCR reaches $2.45\times$ speedup at 30 iterations as repeated
static encoding becomes more prominent.
}
\label{fig:scr_scaling}
\end{figure*}

\begin{figure*}[t]
\centering
\includegraphics[width=1.0\textwidth]{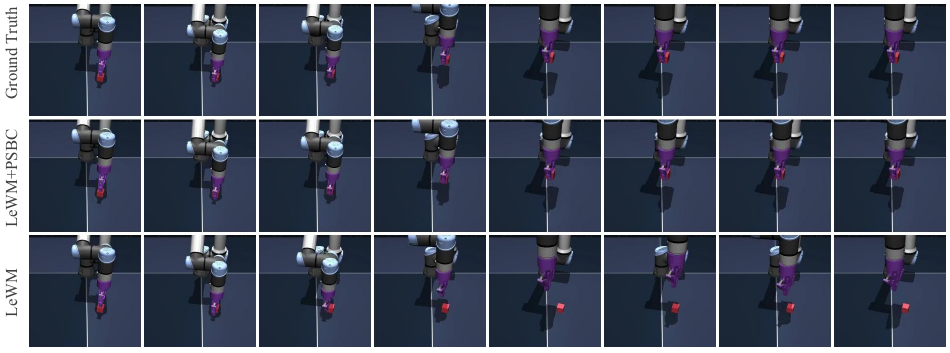}   
\vspace{-0.9em}
\caption{
\textbf{Paired qualitative comparison on Cube.}
From the same initial state and goal, LeWM+PSBC succeeds while
Full LeWM fails; Ground Truth is shown as reference, and columns
progress from left to right.
}
\label{fig:vs}
\end{figure*}

\section{Method}
\label{sec:method}
\vspace{-0.5em}

\noindent\textbf{Framework Overview.}
We propose \textbf{SufficientPlan}, a deployment framework that reduces
decision-time planning computation without retraining the world model or
modifying its native planning objective.
Given a frozen world model \(M\), its planner \(\mathcal{P}\), a
development set \(\mathcal{D}_{\mathrm{dev}}\), the original Full
planning budget \(B_{\mathrm{Full}}\), and an allowed performance
degradation \(\delta\), SufficientPlan first applies
\textbf{Paired Sequential Budget Certification (PSBC)} to identify a
lower model--task-specific operating budget.
The resulting budget is frozen before held-out evaluation.
During planning, \textbf{Static-Context Reuse (SCR)} further removes
repeated observation-history and goal encoding within each solve.
Thus, PSBC reduces candidate search, whereas SCR eliminates redundant
computation inside the retained search.

\noindent\textbf{Paired Outcome Formulation.}
Let \(Y_e(B)\in\{0,1\}\) denote the binary closed-loop outcome obtained on
development episode \(e\) using planning budget \(B\).
For each tested budget proposal \(B_i\), PSBC evaluates \(B_i\) and
\(B_{\mathrm{Full}}\) on the same episode, initial state, and goal.
Their paired outcome difference is
\begin{equation}
D_e(B_i)
=
Y_e(B_{\mathrm{Full}})
-
Y_e(B_i)
\in\{-1,0,1\}.
\label{eq:paired_outcome}
\end{equation}
Here, \(D_e(B_i)=1\) denotes a \emph{broken} episode that succeeds under
Full planning but fails under \(B_i\), whereas \(D_e(B_i)=-1\) denotes a
\emph{rescued} episode that fails under Full planning but succeeds under
\(B_i\).
A value of zero indicates identical binary outcomes.

The population-level degradation and its estimate from \(n\) paired
development episodes are
\begin{equation}
\begin{aligned}
\Delta(B_i)
&=
\mathbb{E}\!\left[D_e(B_i)\right]
=
\mathrm{SR}(B_{\mathrm{Full}})
-
\mathrm{SR}(B_i), \\
\widehat{\Delta}_{i,n}
&=
\frac{1}{n}
\sum_{e=1}^{n} D_e(B_i).
\end{aligned}
\label{eq:paired_degradation}
\end{equation}
Although \(\widehat{\Delta}_{i,n}\) equals the difference between the two
empirical success rates, retaining the episode-level pairing preserves
their covariance and supports paired uncertainty quantification.

\noindent\textbf{Sequential Budget Certification.}
Let \(\mathcal{B}=\{B_1,\ldots,B_m\}\) denote a predefined sequence of
candidate budgets in the lower-budget regime, and let
\(\mathcal{N}\) denote the permitted sequential inspection times.
For each \(B_i\) and \(n\in\mathcal{N}\), let
\(\mathcal{D}_{i,n}=\{D_e(B_i)\}_{e=1}^{n}\) denote the observed paired
outcomes.
PSBC computes the posterior probability that the population degradation
of \(B_i\) lies within the allowed tolerance:
\begin{equation}
q_{i,n}
=
\Pr\!\left(
    \Delta(B_i)\leq\delta
    \mid
    \mathcal{D}_{i,n}
\right).
\label{eq:psbc_posterior_probability}
\end{equation}
The posterior is constructed from the rescued, broken, and unchanged
episode counts under a Dirichlet model.
The full posterior construction is provided in the supplementary
material.

Let \(\tau_{\mathrm{acc}}\) and \(\tau_{\mathrm{rej}}\) denote the
acceptance and rejection thresholds.
PSBC applies
\begin{equation}
\operatorname{Decision}(B_i,n)
=
\begin{cases}
\textsc{Sufficient},
& q_{i,n}\geq\tau_{\mathrm{acc}},\\[2pt]
\textsc{Insufficient},
& 1-q_{i,n}\geq\tau_{\mathrm{rej}},\\[2pt]
\textsc{Uncertain},
& \text{otherwise}.
\end{cases}
\label{eq:psbc_decision}
\end{equation}
A sufficient proposal is certified.
An uncertain proposal receives additional paired episodes at the same
budget, whereas an insufficient proposal is replaced by the next budget
in the proposal sequence.
For each additional episode, the corresponding Full-budget outcome is
evaluated or retrieved from a precomputed reference, preserving the
paired comparison.
If the development episodes are exhausted before a proposal can be
certified, PSBC proceeds conservatively to the next proposal and
ultimately falls back to \(B_{\mathrm{Full}}\) if no reduced budget is
certified.
PSBC thus certifies search sufficiency from paired closed-loop outcomes
rather than action agreement (Fig.~\ref{fig:psbc}).

\noindent\textbf{Posterior Certification Guarantee.}
The posterior decision rule in Eq.~\ref{eq:psbc_decision} provides a
direct probabilistic interpretation for every budget accepted by PSBC.

\begin{proposition}[Posterior certification criterion]
\label{prop:psbc_validity}
If PSBC certifies a budget \(B^\star=B_i\) at inspection time \(n\), then
\begin{equation}
\Pr\!\left(
    \Delta(B^\star)\leq\delta
    \mid
    \mathcal{D}_{i,n}
\right)
\geq
\tau_{\mathrm{acc}}.
\label{eq:psbc_certification_guarantee}
\end{equation}
\end{proposition}

\begin{proof}
Certification requires
\(q_{i,n}\geq\tau_{\mathrm{acc}}\).
By the definition of \(q_{i,n}\) in
Eq.~\ref{eq:psbc_posterior_probability},
\begin{equation}
q_{i,n}
=
\Pr\!\left(
    \Delta(B^\star)\leq\delta
    \mid
    \mathcal{D}_{i,n}
\right),
\end{equation}
which directly gives
Eq.~\ref{eq:psbc_certification_guarantee}.
\end{proof}

Proposition~\ref{prop:psbc_validity} is a posterior statement under the
specified paired-outcome model.
It does not imply simultaneous frequentist coverage over all budgets and
inspection times, nor does it guarantee identical empirical success rates
on every finite held-out sample.

\noindent\textbf{Budget Proposal and Deployment.}
We define planning budget as the total number of candidate action
sequences scored in one planning solve.
If a planner evaluates \(N_k\) candidates at refinement iteration \(k\),
the budget of proposal \(B_i\) is
\begin{equation}
B_i
=
\sum_{k=1}^{K_i}N_k,
\qquad
B_i=NK_i
\quad\text{when }N_k=N.
\label{eq:planning_budget}
\end{equation}
The second expression applies to the fixed-population CEM configurations
used in our experiments.

PSBC examines a predefined proposal sequence beginning in the
lower-budget regime and returns the first proposal classified as
\textsc{Sufficient}, denoted by \(B^\star\).
This defines a certified operating point rather than a globally minimal
budget, and does not assume monotonic closed-loop performance with
increasing planning budget.

After calibration, the deployment budget is frozen as
\begin{equation}
B_{\mathrm{deploy}}
=
B^\star.
\label{eq:frozen_budget}
\end{equation}
Held-out outcomes are never used to revise this budget.
PSBC changes only the amount of deployed search; the world-model
parameters, planning objective, per-iteration candidate-generation rule,
cost function, and optimizer update remain unchanged.

\noindent\textbf{Static-Context Reuse.}
PSBC reduces candidate evaluation, but repeated computation may remain
within each retained planning solve.
Let \(h\) denote the current observation history, \(g\) the goal
observation, and \(\mathcal{A}^{(k)}\) the candidate action population at
refinement iteration \(k\).
A conventional latent planner evaluates this population through
\begin{equation}
J_{\mathrm{orig}}^{(k)}
=
\mathcal{C}\!\left(
    E_h(h),
    E_g(g),
    E_a(\mathcal{A}^{(k)})
\right),
\label{eq:original_cost}
\end{equation}
where \(E_h\) and \(E_g\) encode the observation history and goal,
\(E_a\) encodes candidate actions, and \(\mathcal{C}\) includes latent
prediction and terminal-cost evaluation.

Within one planning solve, \(h\) and \(g\) remain fixed across all \(K\)
refinement iterations, whereas \(\mathcal{A}^{(k)}\) changes.
Nevertheless, Eq.~\ref{eq:original_cost} recomputes \(E_h(h)\) and
\(E_g(g)\) at every iteration.
As illustrated in Fig.~\ref{fig:scr_mechanism}, SCR moves these
solve-invariant operations outside the iterative loop:
\begin{equation}
z_h=E_h(h),
\qquad
z_g=E_g(g).
\label{eq:scr_cache}
\end{equation}
The cached representations are then reused in every candidate evaluation:
\begin{equation}
J_{\mathrm{SCR}}^{(k)}
=
\mathcal{C}\!\left(
    z_h,
    z_g,
    E_a(\mathcal{A}^{(k)})
\right).
\label{eq:scr_cost}
\end{equation}

The cache is local to one planning solve and is discarded before the next
environment step.
SCR does not cache candidate-dependent quantities: candidate sampling,
action encoding, latent rollout, terminal-cost evaluation, and optimizer
updates are still executed for every candidate population.
SCR therefore changes the placement of solve-static computation rather
than the mathematical definition of candidate evaluation.

\noindent\textbf{Computation Reduction and Planning Equivalence.}
For the planners considered here, each refinement iteration performs one
observation-history encoding and one goal encoding.
The original planner therefore performs \(2K\) solve-static encoding
operations over \(K\) iterations, whereas SCR performs only two.
The resulting encoder-call reduction is
\begin{equation}
R_{\mathrm{enc}}(K)
=
1-\frac{2}{2K}
=
1-\frac{1}{K}.
\label{eq:encoding_reduction}
\end{equation}
For example, Eq.~\ref{eq:encoding_reduction} gives an 87.5\% reduction
for \(K=8\) and a 96.7\% reduction for \(K=30\).
This quantity measures the reduction in solve-static encoder calls, not
the reduction in total planner FLOPs, latency, or candidate evaluations.

Under deterministic evaluation-mode encoders, identical initial optimizer
and random states, and unchanged candidate-dependent computation,
Eqs.~\ref{eq:original_cost} and \ref{eq:scr_cost} use identical
observation and goal representations.
Consequently,
\begin{equation}
J_{\mathrm{SCR}}^{(k)}
=
J_{\mathrm{orig}}^{(k)}
\qquad
\forall k\in\{1,\ldots,K\}.
\label{eq:scr_cost_equivalence}
\end{equation}
Applying the same optimizer update to identical candidate costs preserves
the subsequent optimizer states and sampling distributions, yielding
\begin{equation}
a_{\mathrm{SCR}}^\star
=
a_{\mathrm{orig}}^\star.
\label{eq:scr_action_equivalence}
\end{equation}
The complete induction argument and numerical equivalence tests are
provided in the supplementary material.
SCR therefore reduces solve-static computation while preserving
fixed-budget control behavior.

\noindent\textbf{Combined Deployment.}
Given the certified operating budget \(B^\star\), SufficientPlan executes
the original planner with SCR enabled:
\begin{equation}
a_t
=
\mathcal{P}_{\mathrm{SCR}}
\left(
    M,h_t,g;B^\star
\right).
\label{eq:sufficientplan_deployment}
\end{equation}
Neither component retrains the world model or modifies the planning
objective, candidate-generation rule, or optimizer update.

\begin{table*}[t]
\centering
\caption{
\textbf{Incremental ablation on the frozen Push-T development set.}
PSBC reduces the deployed search budget while retaining near-Full success,
and SCR preserves the resulting budget and control outcome.
}
\label{tab:psbc_ablation}
\small
\setlength{\tabcolsep}{7pt}
\renewcommand{\arraystretch}{1.12}

\resizebox{\textwidth}{!}{
\begin{tabular}{llcccccc}
\toprule
\textbf{Backbone}
& \textbf{Configuration}
& \textbf{PSBC}
& \textbf{SCR}
& \textbf{Scored Seq./Solve}
& \textbf{Budget (\% Full)}
& \textbf{SR (\%)}
& $\boldsymbol{\Delta}$\textbf{SR (pp)}
\\
\midrule

\multirow{3}{*}{INTACT}
& Full Actor-CEM
& \xmark
& \xmark
& 9,000
& 100.0
& 96
& 0
\\

& + PSBC
& \cmark
& \xmark
& 2,400
& 26.7
& 96
& 0
\\

& + PSBC + SCR
& \cmark
& \cmark
& 2,400
& 26.7
& 96
& 0
\\

\midrule

\multirow{3}{*}{Fast-LeWM}
& Full CEM
& \xmark
& \xmark
& 9,000
& 100.0
& 90
& 0
\\

& + PSBC
& \cmark
& \xmark
& 2,400
& 26.7
& 89
& $-1$
\\

& + PSBC + SCR
& \cmark
& \cmark
& 2,400
& 26.7
& 89
& $-1$
\\

\bottomrule
\end{tabular}
}
\end{table*}

\section{Experiments}
\label{sec:experiments}
\vspace{-0.3em}

We evaluate whether PSBC reduces deployment search across model--task
pairs, whether the calibrated budgets transfer to held-out episodes, and
whether SCR reduces planning latency without changing control outcomes.
Implementation details, numerical equivalence tests, and additional
analyses are provided in the supplementary material.

\subsection{Experimental Setup}
\label{sec:experimental_setup}

\noindent\textbf{Backbones and tasks.}
We evaluate SufficientPlan on LeWM~\citep{lewm},
Fast-LeWM~\citep{fast_lewm}, PRISM~\citep{PRISM},
C-JEPA~\citep{causal_jepa}, and INTACT~\citep{intact}.
We retain their native planners: CEM for LeWM, Fast-LeWM, and C-JEPA;
PoG-MPPI for PRISM; and Actor-CEM for INTACT.
The evaluation covers Push-T~\citep{pusht}, Cube~\citep{cube}, Reacher~\citep{reacher}, and TwoRoom~\citep{tworoom}.
All checkpoints are frozen during PSBC calibration and held-out
evaluation; SufficientPlan requires no additional world-model training
or fine-tuning.

\noindent\textbf{Protocol and metrics.}
PSBC calibrates each model--task pair on a frozen development set by
comparing candidate budgets with Full planning on paired episodes.
The resulting budget $B^\star$ is frozen before episode-disjoint held-out
evaluation and is not adjusted using held-out outcomes.
We report closed-loop success rate (SR), budget relative to Full planning,
and
\[
\Delta \mathrm{SR}
=
\mathrm{SR}(B^\star)-\mathrm{SR}(B_{\mathrm{Full}}).
\]
For SCR, we additionally report matched-budget planning latency, speedup,
and episode-level outcome flips.

\subsection{Budget Certification}
\label{sec:budget_results}

Tab.~\ref{tab:planning_budget_success} compares Full planning with
PSBC-calibrated deployment across 16 backbone--task pairs.
On average, PSBC uses only 30.0\% of the corresponding Full budget,
reducing candidate evaluations by 70.0\%.
In 14 of the 16 held-out configurations, the resulting performance is
no more than 2 percentage points below Full planning.

The certified budgets vary substantially across model--task pairs,
ranging from 3.3\% to 80.0\% of Full planning.
For example, PSBC selects 13.3\% for LeWM on Cube but 80.0\% on Reacher,
highlighting the inefficiency of a globally fixed search budget.
The same calibration procedure applies to CEM, PoG-MPPI, and Actor-CEM
without modifying their native optimizer updates.

Fig.~\ref{fig:vs} further shows a paired qualitative example on Cube,
where LeWM+PSBC succeeds from the same initial condition for which
Full LeWM fails.
The main held-out degradations occur on Fast-LeWM/Push-T and
C-JEPA/Push-T, with decreases of 6 and 3 percentage points, respectively.
These results are reported without held-out recalibration, indicating
that PSBC improves the empirical performance--computation trade-off
without assuming that reduced-budget planning must match Full planning
on every model--task pair.

\subsection{Static-Context Reuse}
\label{sec:scr_results}

Fig.~\ref{fig:scr} compares SCR with uncached planning under matched
budgets, checkpoints, episodes, candidate samples, and optimizer
configurations.
Across the evaluated settings, SCR preserves the episode-success vectors
with zero outcome flips while achieving up to $2.18\times$ planning
speedup.
The gain varies across settings because the fraction of runtime spent on
repeated static encoding depends on the backbone, task, and search depth.

Fig.~\ref{fig:scr_scaling} evaluates Fast-LeWM/Push-T over budgets from
300 to 9,000 scored sequences.
As refinement depth increases from 1 to 30 iterations, the analytical
reduction in static-encoder calls grows from 0\% to 96.7\%.
Repeated measurements show no reliable speedup at one iteration, where
no cross-iteration reuse is available, but SCR reaches $2.45\times$
speedup at 30 iterations.

\subsection{Further analysis of PSBC}
Appendix~\ref{sec:additional_psbc_analysis} further compares PSBC with
empirical and fixed-sample posterior selection and analyzes its sequential
inspection behavior.
In the PRISM/Cube case study, paired posterior certification rejects a
budget that would be accepted from empirical success rates alone, while
early rejection reduces candidate-side calibration evaluations.
The sensitivity analysis further shows that stricter certification
settings require more evidence or leave the tested proposal uncertified
within the available calibration horizon.

\subsection{Ablation Study}
\label{sec:ablation}

Tab.~\ref{tab:psbc_ablation} incrementally evaluates PSBC and SCR on the
frozen Push-T development set.
For INTACT, PSBC reduces the search budget from 9,000 to 2,400 scored
sequences while preserving the 96\% success rate.
For Fast-LeWM, it reduces the same budget to 26.7\% of Full planning with
a 1-point decrease in success rate.
Adding SCR leaves the selected budget and success outcomes unchanged for
both backbones, as expected from its behavior-preserving computation reuse.
Its latency benefits are reported separately in Fig.~\ref{fig:scr}.

\section{Conclusion}
\label{sec:conclusion}
\vspace{-0.3em}

In this work, we show that world-model planners can produce different
actions under reduced and Full planning budgets despite achieving the
same aggregate task performance, and that they can repeatedly encode
solve-invariant observations and goals.
To address these two sources of overcomputation, we propose {SufficientPlan}, a simple yet effective deployment framework that
requires neither world-model retraining nor changes to the native planner.
PSBC uses paired closed-loop evidence to certify a lower
model--task-specific search budget, while SCR reuses solve-static
representations across planning iterations.
Experiments across multiple backbones, planners, and visual-control tasks
show that SufficientPlan substantially reduces candidate evaluations and
planning latency while retaining competitive control performance.
These results highlight outcome sufficiency and computation reuse as
complementary principles for efficient world-model planning.

\section*{AI Use Statement}
\label{sec:ai_use}
AI tools were used to polish the language and improve readability.
The authors take full responsibility for the paper and its results.

\section*{Ethics Statement}
This work evaluates planning efficiency on simulated visual-control
benchmarks. It does not involve human-subject experiments or the
collection of personal data. We do not anticipate direct ethical risks
from the reported simulation experiments. Any physical deployment should
follow the safety requirements of the intended application.

\section*{Reproducibility Statement}
We provide detailed descriptions of the model architecture, optimization objectives, training schedule, hyperparameters, datasets, and evaluation protocols in the main paper and supplementary material. All experiments use consistent settings unless otherwise specified. We plan to release the implementation and configuration files to facilitate reproduction of the reported results.

\bibliography{iclr2027_conference}
\bibliographystyle{iclr2027_conference}

\newpage
\appendix

\renewcommand\thefigure{\Alph{figure}} 
\setcounter{figure}{0}
\renewcommand\thetable{\Alph{table}} 
\setcounter{table}{0}
\renewcommand{\thesection}{\Alph{section}}
\setcounter{section}{0}
\definecolor{myorange}{RGB}{230,145,56}

\begin{tcolorbox}[
    enhanced,
    breakable,
    colback=suppbg,
    colframe=suppborder,
    boxrule=0.4pt,
    arc=3mm,
    left=4mm,
    right=4mm,
    top=3mm,
    bottom=3mm
]

\begin{center}
{\Large\bfseries Supplementary Material}
\end{center}

\vspace{0.3em}

\noindent
\textbf{This supplementary material is organized as follows:}

\begin{itemize}
\setlength{\itemsep}{1pt}
\setlength{\parsep}{0pt}
\setlength{\parskip}{0pt}
\setlength{\topsep}{3pt}

\item Appendix~\ref{sec:implementation_details} provides dataset,
implementation, calibration, and evaluation details.

\item Appendix~\ref{app:theoretical_analysis} presents the posterior
construction of PSBC and the theoretical analysis of SCR.

\item Appendix~\ref{sec:supplementary_experiments} reports additional
experimental results, including model--task budget variation,
multi-seed evaluation, deployment ablations, qualitative comparisons,
and numerical verification of SCR.

\item Appendix~\ref{sec:additional_psbc_analysis} provides further
analysis of PSBC, including comparisons with simpler selection criteria,
sequential inspection efficiency, and sensitivity to certification
parameters.

\item Appendix~\ref{sec:limitations} discusses the scope and limitations
of SufficientPlan.


\end{itemize}

\end{tcolorbox}

\vspace{0.8em}

\section{Implementation Details}
\label{sec:implementation_details}
\subsection{Datasets and Evaluation Protocol}

We evaluate SufficientPlan on four visual-control tasks: Push-T, Cube,
Reacher, and TwoRoom.
Push-T requires a planar agent to move a T-shaped object to a target
configuration, Cube evaluates vision-based robotic manipulation, Reacher
requires controlling a planar arm toward a target, and TwoRoom evaluates
goal-directed navigation across connected rooms.
We use the dataset and preprocessing protocol associated with each
world-model implementation.

For every model--task pair, we construct a frozen development split for
budget calibration and an episode-disjoint held-out split for final
evaluation.
Unless otherwise stated, each calibration uses at most 100 development
episodes, and each final evaluation uses 100 episode-disjoint held-out
episodes.
PSBC uses only the development split to determine the deployment budget.
The resulting budget is fixed before held-out evaluation and is not
adjusted according to held-out outcomes.
Within every paired comparison, the candidate-budget and Full planners
are evaluated using the same episodes, initial states, and goals.

We report episode-level success rate (SR), with
\begin{equation}
\Delta\mathrm{SR}
=
\mathrm{SR}(B^\star)
-
\mathrm{SR}(B_{\mathrm{Full}}),
\label{eq:implementation_delta_sr}
\end{equation}
where $B^\star$ is the budget returned by PSBC and
$B_{\mathrm{Full}}$ is the corresponding Full-planning budget.
Both success rates in Eq.~\eqref{eq:implementation_delta_sr} are
computed on the same evaluation split.
A negative $\Delta\mathrm{SR}$ indicates a decrease relative to Full
planning, whereas a positive value indicates higher empirical success
under the calibrated budget.

For the multi-seed evaluation, we use evaluation seeds
$\{0,1,42\}$.
We report the mean success rate across the three seeds, with error bars
showing the sample standard deviation.
The development and held-out evaluations remain separated within each
seed.

\begin{figure*}[t]
\centering
\includegraphics[width=1.0\textwidth]
{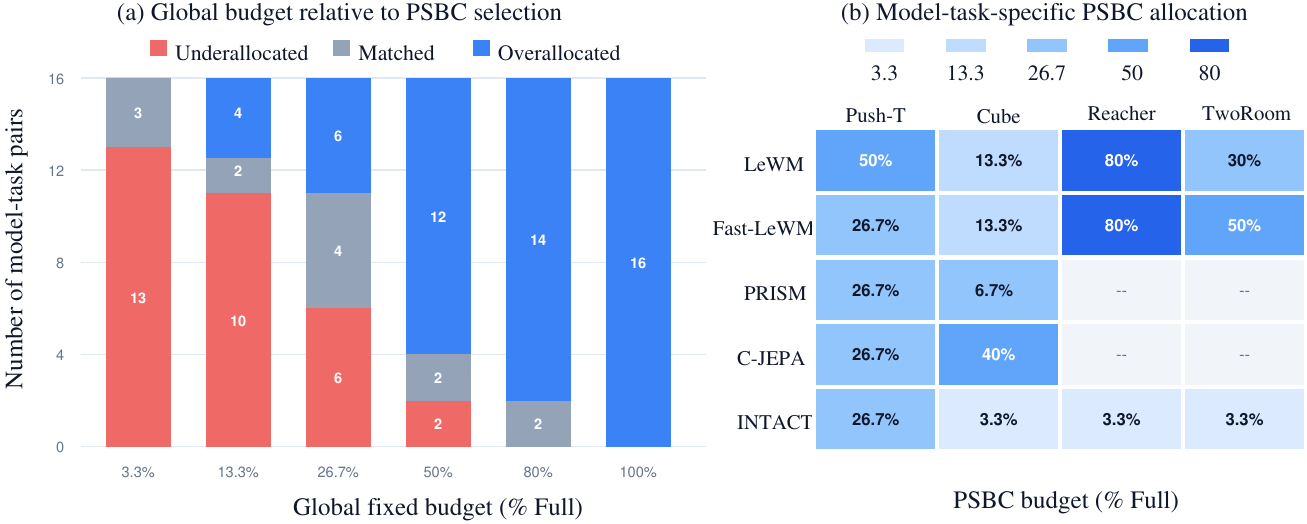}
\vspace{-0.8em}
\caption{
\textbf{Global fixed budgets versus PSBC-selected budgets.}
(a) Numbers of model--task pairs for which each global budget is below,
equal to, or above the PSBC-selected budget.
(b) Selected budget fractions across the 16 evaluated pairs.
}
\label{fig:global_budget_vs_psbc}
\end{figure*}

\subsection{Backbones and Pretrained Checkpoints}

We evaluate LeWM~\citep{lewm}, Fast-LeWM~\citep{fast_lewm},
PRISM~\citep{PRISM}, C-JEPA~\citep{causal_jepa}, and
INTACT~\citep{intact}.
We use officially released pretrained checkpoints whenever they are
available.
For model--task combinations without released checkpoints, we train the
corresponding model using the official implementation and recommended
training configuration.
In particular, the C-JEPA model evaluated on Cube is trained using the
official C-JEPA training protocol because a pretrained checkpoint for
this setting is not provided.
All evaluated checkpoints are frozen during budget calibration and
held-out evaluation; SufficientPlan introduces no additional world-model
training or fine-tuning.

We retain the native planner of each backbone.
LeWM, Fast-LeWM, and C-JEPA use CEM, PRISM uses PoG-MPPI, and INTACT uses
Actor-CEM with its pretrained action prior.
Except for the search budget selected by PSBC and the computation reuse
introduced by SCR, the original planner configuration and update rules
remain unchanged.

\begin{table}[t]
\centering
\caption{
\textbf{Deployment efficiency on PRISM/Cube.}
PSBC reduces per-episode planning time relative to Full planning.
}
\label{tab:deployment_efficiency}
\small
\setlength{\tabcolsep}{7pt}
\renewcommand{\arraystretch}{1.10}

\begin{tabular}{lr}
\toprule
\textbf{Quantity} & \textbf{Value} \\
\midrule
Full planning budget & 9,000 \\
PSBC deployment budget & 600 \\
Full time per episode & 34.28 s \\
PSBC time per episode & 10.93 s \\
Time saving per episode & 23.35 s \\
\bottomrule
\end{tabular}
\end{table}

\subsection{Planning and Evaluation Settings}

We define the planning budget as the number of candidate action sequences
scored during one planning solve.
For CEM-based planners with a population of \(N\) candidates and \(K\)
refinement iterations, the budget is
\begin{equation}
B=NK.
\label{eq:planning_budget_implementation}
\end{equation}
The corresponding relative budget is
\(100B/B_{\mathrm{Full}}\).

In our fixed-population CEM experiments, we use \(N=300\) candidates per
iteration and vary the number of refinement iterations.
The feasible CEM budget grid is
\[
\mathcal{B}_{\mathrm{CEM}}
=
\{300,600,\ldots,9000\},
\]
corresponding to \(K\in\{1,\ldots,30\}\).
For example, 8 iterations correspond to 2,400 scored sequences, whereas
the 30-iteration Full planner scores 9,000 sequences.
For PoG-MPPI, candidate budgets follow the native sampling
parameterization of the original planner, while the reported budget
continues to denote the total number of candidate action sequences scored
per solve.
Planner-specific parameters other than the evaluated search budget,
including the planning horizon, candidate-generation procedure, elite
selection, and optimizer update, are inherited from the corresponding
native implementation.

\noindent\textbf{PSBC calibration settings.}
For all model--task pairs, we set the allowed success-rate degradation to
\(\delta=0.02\), corresponding to a tolerance of two percentage points
relative to Full planning.
For a fixed model--task pair and candidate budget, paired episode outcomes
are modeled as exchangeable categorical draws over rescued, broken, and
unchanged outcomes.
We use the symmetric
\(\operatorname{Dirichlet}
(\tfrac{1}{2},\tfrac{1}{2},\tfrac{1}{2})\)
prior and estimate the posterior non-inferiority probability using
\(10^5\) Monte Carlo samples.
Posterior sampling uses deterministic seeds fixed before calibration so
that repeated analysis of the same paired outcomes produces the same
decision.

PSBC inspects the posterior after
\[
\mathcal{N}
=
\{10,20,40,60,80,100\}
\]
paired calibration episodes and uses
\(\tau_{\mathrm{acc}}=\tau_{\mathrm{rej}}=0.95\).
A proposal is certified when its posterior probability of remaining
within the tolerance is at least 0.95, rejected when the posterior
probability of exceeding the tolerance is at least 0.95, and otherwise
retained as uncertain for the next inspection point.

For the fixed-population CEM planners, PSBC considers the feasible
budgets in ascending order,
\[
300,600,\ldots,9000.
\]
An uncertain proposal receives additional paired episodes up to the next
inspection point.
If it remains unresolved after the final inspection point, or is
classified as insufficient, PSBC advances to the next larger proposal.
For PoG-MPPI, the same lower-to-higher policy is applied to the feasible
budgets induced by its native sampling parameterization.
The first proposal satisfying the certification criterion is returned as
\(B^\star\).
This is the first certified proposal under the predefined policy rather
than a claim of global minimality over all feasible planning budgets.
If no reduced proposal is certified within the available development
episodes, PSBC falls back to \(B_{\mathrm{Full}}\).

The tolerance \(\delta\) applies to the posterior population-level
degradation under the development distribution.
It is not a hard constraint on the empirical success-rate difference
observed in every finite held-out sample.
Accordingly, finite held-out evaluations can occasionally exhibit a
degradation larger than two percentage points due to sampling variation
or development-to-held-out variation, without any held-out recalibration.

\noindent\textbf{SCR evaluation settings.}
We evaluate SCR against the corresponding uncached planner under matched
planning budgets, checkpoints, episodes, candidate samples, and optimizer
configurations.
The SCR cache is constructed once at the beginning of each planning solve,
reused across its refinement iterations, and discarded before the next
environment step.
Cache construction is included in the reported planning latency.
Candidate sampling, candidate-action encoding, latent rollout,
terminal-cost evaluation, and optimizer updates remain inside the timed
planning computation and are not cached.

All reported latency experiments are conducted on a single NVIDIA RTX
A5000 GPU per run.
We report mean solver wall-clock time per batched planning solve, latency
reduction, speedup, and episode-level outcome flips.
The timed region covers the complete planner solve and excludes dataset
loading, environment rendering, metric aggregation, and video generation.
The uncached and SCR-enabled configurations use the same hardware,
checkpoint, evaluation episodes, candidate budget, and numerical
precision.

For the Fast-LeWM/Push-T budget-scaling experiment, each configuration is
evaluated on the same 20 paired episodes.
To reduce sensitivity to execution order, we use an order-balanced
uncached--SCR--SCR--uncached measurement sequence and aggregate the
repeated measurements for each method.
The reported timing protocol does not assume that the analytical
encoder-call reduction is equal to the measured latency reduction.

Numerical equivalence is evaluated separately at the Full
9,000-sequence budget using matched candidate samples and
random-number-generator states.
We compare candidate populations, candidate costs, CEM distribution
updates, selected actions, solver random-number-generator states, model
buffers, and the total number of scored candidate sequences.

\begin{figure*}[t]
\centering
\includegraphics[width=1.0\textwidth]{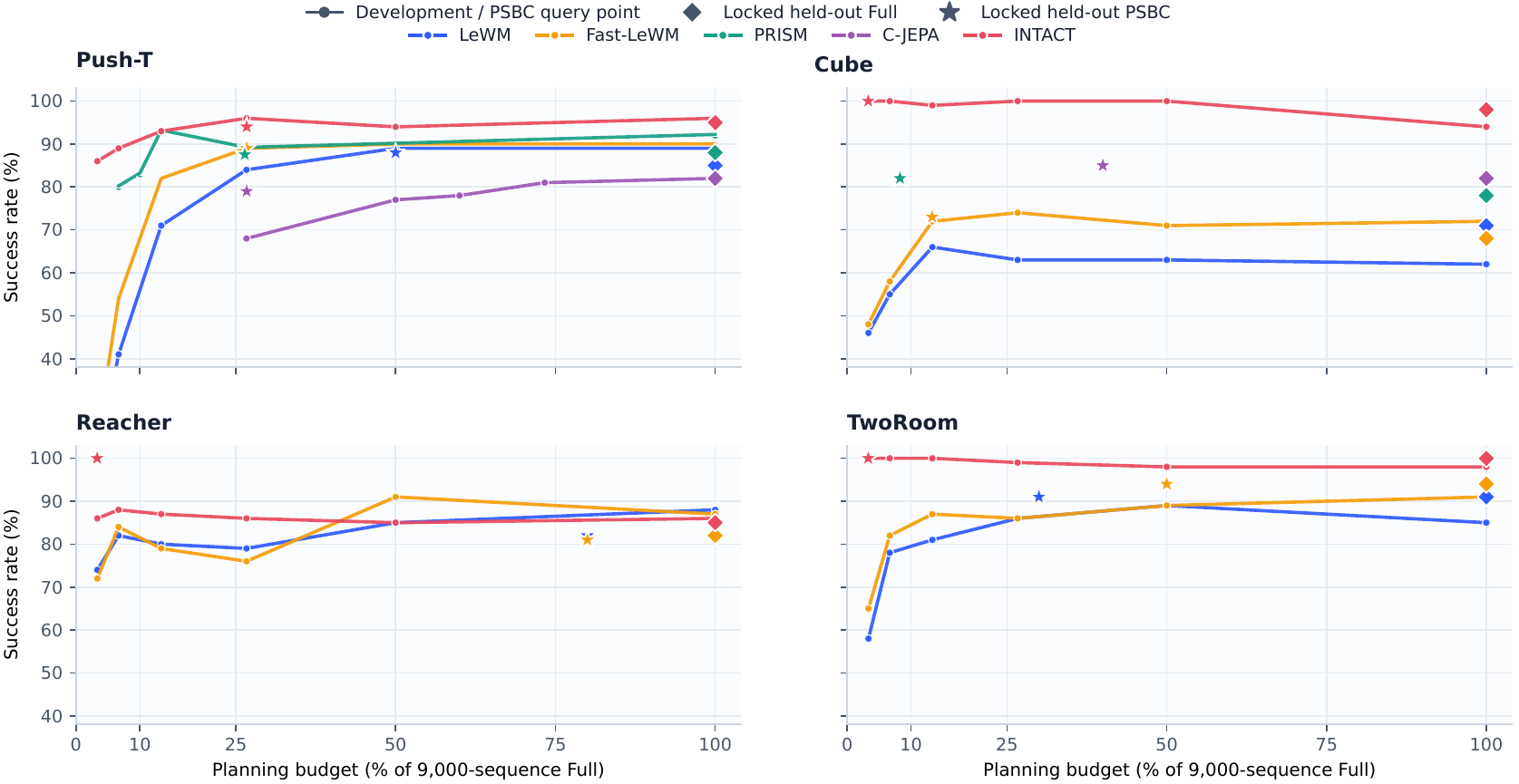}
\caption{
\textbf{Success--budget profiles across world-model planners and tasks.}
Lines show development-set success rates across planning budgets, while
diamonds and stars denote held-out Full and PSBC results, respectively.
Development curves and held-out markers correspond to distinct evaluation
splits.
}
\label{fig:diff_backbone}
\end{figure*}

\begin{table*}[t]
\centering
\caption{
\textbf{Posterior evolution after PSBC certification on PRISM.}
$q_{100}$ retrospectively evaluates the locked 100-episode set and is
not used in the original stopping decision.
}
\label{tab:early_certificate_audit}
\small
\setlength{\tabcolsep}{5pt}
\renewcommand{\arraystretch}{1.08}
\begin{tabular}{llrrrrr}
\toprule
\textbf{Task} & \textbf{Seed} & $\boldsymbol{B^\star}$ & \textbf{Stop $n$} & $\boldsymbol{q_{B,n}}$ & $\boldsymbol{q_{B,100}}$ & \textbf{SR: $B$/Full} \\
\midrule
Push-T & 0 & 3,300 & 20 & 0.956 & 0.510 & 88/90 \\
Push-T & 1 & 1,800 & 20 & 0.956 & 0.799 & 88/87 \\
Cube & 0 & 900 & 80 & 0.972 & 0.971 & 88/82 \\
Cube & 1 & 600 & 100 & 0.981 & 0.981 & 92/85 \\
\bottomrule
\end{tabular}
\end{table*}

\section{Theoretical Analysis}
\label{app:theoretical_analysis}

\subsection{PSBC Posterior Construction and Sequential Decision}
\label{app:psbc_posterior}

For a candidate budget \(B_i\), each paired episode produces one of three
mutually exclusive outcomes:
\emph{rescued}, \emph{broken}, or \emph{unchanged}.
After \(n\) paired episodes, let
\begin{equation}
R_{i,n}
=
\sum_{e=1}^{n}
\mathbb{I}\!\left[D_e(B_i)=-1\right],
\qquad
B_{i,n}
=
\sum_{e=1}^{n}
\mathbb{I}\!\left[D_e(B_i)=1\right],
\label{eq:psbc_rescued_broken_counts}
\end{equation}
and
\begin{equation}
U_{i,n}
=
n-R_{i,n}-B_{i,n},
\label{eq:psbc_unchanged_count}
\end{equation}
where \(R_{i,n}\), \(B_{i,n}\), and \(U_{i,n}\) denote the numbers of
rescued, broken, and unchanged episodes, respectively.
Here, the subscript \(B_{i,n}\) denotes the broken-episode count and
should not be confused with the candidate planning budget \(B_i\).

Let
\begin{equation}
\boldsymbol{p}_i
=
\left(
p_{i,R},p_{i,B},p_{i,U}
\right)
\end{equation}
denote the population probabilities of these three outcomes.
We place the symmetric Jeffreys-type prior
\begin{equation}
\boldsymbol{p}_i
\sim
\operatorname{Dirichlet}
\left(
\tfrac{1}{2},
\tfrac{1}{2},
\tfrac{1}{2}
\right).
\label{eq:psbc_dirichlet_prior}
\end{equation}
Given the paired outcomes observed after \(n\) episodes, conjugacy yields
\begin{equation}
\boldsymbol{p}_i
\mid
\mathcal{D}_{i,n}
\sim
\operatorname{Dirichlet}
\left(
R_{i,n}+\tfrac{1}{2},
B_{i,n}+\tfrac{1}{2},
U_{i,n}+\tfrac{1}{2}
\right).
\label{eq:psbc_dirichlet_posterior}
\end{equation}

Under the paired-outcome definition in
Eq.~\ref{eq:paired_outcome}, the population success-rate degradation of
\(B_i\) relative to Full planning is
\begin{equation}
\Delta(B_i)
=
p_{i,B}-p_{i,R}.
\label{eq:psbc_posterior_degradation}
\end{equation}
The posterior probability that \(B_i\) satisfies the prescribed
degradation tolerance \(\delta\) is therefore
\begin{equation}
q_{i,n}
=
\Pr\!\left(
p_{i,B}-p_{i,R}\leq\delta
\mid
\mathcal{D}_{i,n}
\right).
\label{eq:psbc_posterior_noninferiority}
\end{equation}
We estimate \(q_{i,n}\) by drawing samples from
Eq.~\ref{eq:psbc_dirichlet_posterior} and computing the fraction that
satisfies Eq.~\ref{eq:psbc_posterior_noninferiority}.

Let \(\tau_{\mathrm{acc}}\) and \(\tau_{\mathrm{rej}}\) denote the
acceptance and rejection thresholds.
PSBC applies the sequential decision rule
\begin{equation}
\operatorname{Decision}(B_i,n)
=
\begin{cases}
\textsc{Sufficient},
&
q_{i,n}\geq\tau_{\mathrm{acc}},\\[2pt]
\textsc{Insufficient},
&
1-q_{i,n}\geq\tau_{\mathrm{rej}},\\[2pt]
\textsc{Uncertain},
&
\text{otherwise}.
\end{cases}
\label{eq:psbc_posterior_decision}
\end{equation}
In our experiments,
\(\tau_{\mathrm{acc}}=\tau_{\mathrm{rej}}=0.95\).
An uncertain candidate receives additional paired episodes at the next
inspection time, whereas an insufficient candidate is removed from
further consideration.
The first sufficient candidate under the predefined budget-proposal
policy is returned as \(B^\star\).

\paragraph{Interpretation.}
If PSBC certifies \(B^\star=B_i\) at inspection time \(n\), then
\begin{equation}
\Pr\!\left(
\Delta(B^\star)\leq\delta
\mid
\mathcal{D}_{i,n}
\right)
\geq\tau_{\mathrm{acc}}.
\label{eq:psbc_posterior_guarantee}
\end{equation}
This is a posterior certification statement under the specified
Dirichlet model.
It should not be interpreted as a simultaneous frequentist coverage
guarantee over all candidate budgets and inspection times.
Likewise, it does not imply identical success rates on every finite
held-out sample or global minimality over the complete feasible budget
space.

The sequential trace in Fig.~\ref{fig:psbc_trace} visualizes
\(q_{i,n}\) as paired outcomes accumulate.
Because each posterior update incorporates additional stochastic
closed-loop outcomes, \(q_{i,n}\) need not change monotonically with
\(n\).

\begin{table*}[t]
\centering
\caption{
\textbf{Budget-selection criteria on PRISM/Cube.}
All candidates use the same 100 paired development episodes with a
2-point tolerance; fixed-$n$ certification requires
$q_{B,100}\geq0.95$.
}
\label{tab:psbc_selection_comparison}
\small
\setlength{\tabcolsep}{5pt}
\renewcommand{\arraystretch}{1.10}
\resizebox{\textwidth}{!}{
\begin{tabular}{rrrrrrrrcc}
\toprule
\textbf{Budget}
& \textbf{\% Full}
& \textbf{SR}
& \textbf{Full SR}
& $\boldsymbol{\Delta}$\textbf{SR}
& \textbf{Rescued}
& \textbf{Broken}
& $\boldsymbol{q_{B,100}}$
& \textbf{Empirical Accept}
& \textbf{Fixed-$n$ Cert.}
\\
\midrule
300
& 3.3
& 83
& 85
& $-2$
& 10
& 12
& 0.503
& \cmark
& \xmark
\\

600
& 6.7
& 92
& 85
& $+7$
& 13
& 6
& 0.981
& \cmark
& \cmark
\\

1,200
& 13.3
& 89
& 85
& $+4$
& 12
& 8
& 0.911
& \cmark
& \xmark
\\

2,400
& 26.7
& 91
& 85
& $+6$
& 10
& 4
& 0.984
& \cmark
& \cmark
\\
\bottomrule
\end{tabular}
}
\end{table*}

\subsection{Exact Planning Equivalence of SCR}
\label{app:scr_equivalence}

\begin{theorem}[Exact planning equivalence of SCR]
\label{thm:scr_equivalence}
Consider a planning solve with \(K\) refinement iterations.
Assume that:

\begin{enumerate}
    \item the observation history \(h\) and goal \(g\) remain fixed
    within the solve;
    \item \(E_h\) and \(E_g\) are deterministic in evaluation mode;
    \item the original and SCR planners begin from identical optimizer
    and random-number-generator states;
    \item candidate generation, candidate-dependent computation, cost
    evaluation, and optimizer updates are unchanged and deterministic
    given their inputs; and
    \item candidate ranking and tie-breaking are deterministic.
\end{enumerate}

Then, in exact arithmetic, the original and SCR planners produce
identical candidate populations, candidate costs, optimizer states, and
selected actions at every refinement iteration.
\end{theorem}

\begin{proof}
SCR computes
\begin{equation}
z_h=E_h(h),
\qquad
z_g=E_g(g)
\end{equation}
before refinement.
Because \(h\) and \(g\) remain fixed and their encoders are deterministic,
the cached representations are identical to those recomputed by the
original planner at every iteration.

We prove equivalence by induction over \(k\).
At \(k=1\), the two planners begin from identical optimizer and
random-number-generator states and therefore generate the same candidate
population \(\mathcal{A}^{(1)}\).
Using identical static representations and unchanged candidate-dependent
computation gives
\begin{align}
J_{\mathrm{SCR}}^{(1)}
&=
\mathcal{C}\!\left(
z_h,z_g,E_a(\mathcal{A}^{(1)})
\right)\\
&=
\mathcal{C}\!\left(
E_h(h),E_g(g),E_a(\mathcal{A}^{(1)})
\right)\\
&=
J_{\mathrm{orig}}^{(1)}.
\end{align}
Applying the same optimizer update to identical candidate costs produces
identical optimizer states for iteration \(2\).

Assume that both planners enter iteration \(k\) with identical optimizer
and random-number-generator states.
They generate the same \(\mathcal{A}^{(k)}\), and the same argument gives
\begin{equation}
J_{\mathrm{SCR}}^{(k)}
=
J_{\mathrm{orig}}^{(k)}.
\end{equation}
The unchanged optimizer update consequently preserves the optimizer and
random-number-generator states for iteration \(k+1\).
By induction, all candidate populations, costs, and optimizer states are
identical for \(k=1,\ldots,K\).
Deterministic ranking and tie-breaking therefore yield
\begin{equation}
a_{\mathrm{SCR}}^\star
=
a_{\mathrm{orig}}^\star.
\end{equation}
\end{proof}

Theorem~\ref{thm:scr_equivalence} is stated in exact arithmetic.
In finite-precision implementations, exact equality additionally requires
matched operation ordering, numerical precision, and deterministic
kernels.
The full-budget numerical contract in
Tab.~\ref{tab:scr_full_contract} verifies the implemented planners
under matched candidate samples and random states.

\begin{figure*}[t]
\centering
\includegraphics[width=1.0\textwidth]{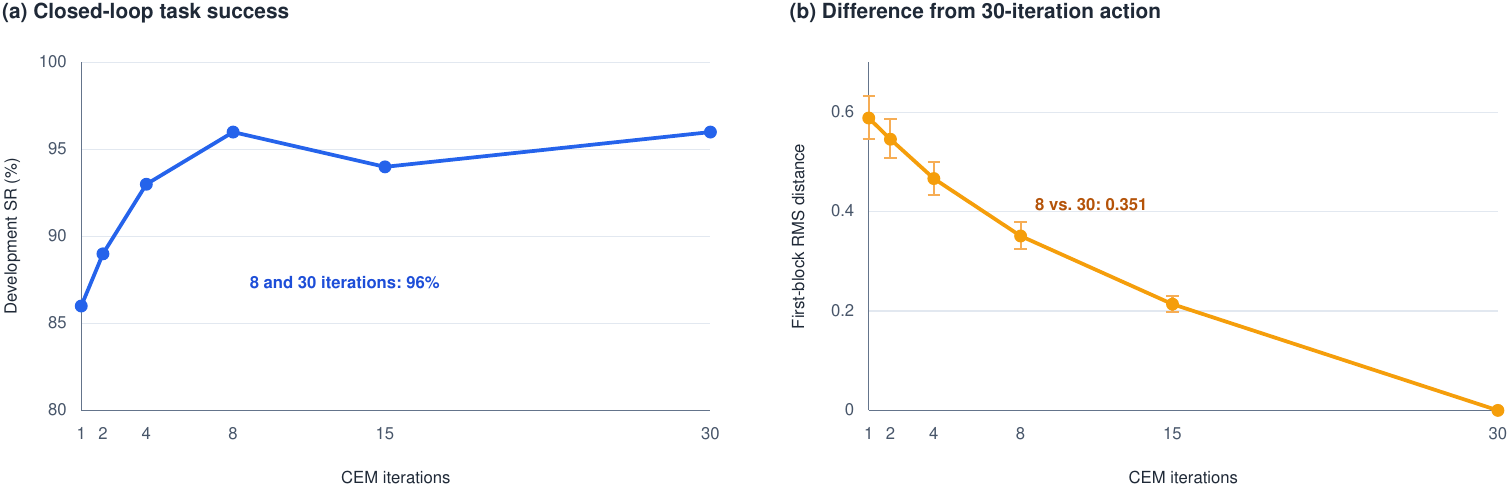}
\caption{
\textbf{Task success and action refinement on INTACT Push-T.}
(a) Development success rate across 100 episodes at different CEM iterations.
(b) Mean RMS distance between the first action block at each iteration
and its 30-iteration counterpart, with paired bootstrap 95\% confidence
intervals. Eight and 30 iterations achieve the same 96\% success rate
despite selecting different actions.
}
\label{fig:action_vs_success}
\end{figure*}

\begin{table*}[t]
\centering
\caption{
\textbf{INTACT deployment ablation on the frozen Push-T development set.}
We isolate the effects of actor-guided initialization, PSBC, and SCR.
}
\label{tab:intact_deployment_comparison}
\small
\setlength{\tabcolsep}{6pt}
\renewcommand{\arraystretch}{1.12}

\resizebox{\textwidth}{!}{
\begin{tabular}{lcccrrrr}
\toprule
\textbf{Deployment}
& \textbf{Actor Prior}
& \textbf{PSBC}
& \textbf{SCR}
& \textbf{Seq./Solve}
& \textbf{Budget (\% Full)}
& \textbf{Resulting SR (\%)}
& $\boldsymbol{\Delta}$\textbf{SR vs Direct}
\\
\midrule

Direct
& \cmark
& \xmark
& \xmark
& 0
& 0.0
& 83
& 0
\\

Guarded-A
& \cmark
& \xmark
& \xmark
& 384
& 4.3
& 88
& $+5$
\\

Pure CEM
& \xmark
& \xmark
& \xmark
& 9,000
& 100.0
& 92
& $+9$
\\

Full Actor-CEM
& \cmark
& \xmark
& \xmark
& 9,000
& 100.0
& \textbf{96}
& \textbf{$+13$}
\\

\midrule

\rowcolor{blue!6}
Actor-CEM + PSBC
& \cmark
& \cmark
& \xmark
& \textbf{2,400}
& \textbf{26.7}
& \textbf{96}
& \textbf{$+13$}
\\

\rowcolor{blue!10}
\textbf{SufficientPlan}
& \cmark
& \cmark
& \cmark
& \textbf{2,400}
& \textbf{26.7}
& \textbf{96}
& \textbf{$+13$}
\\

\bottomrule
\end{tabular}
}

\end{table*}

\subsection{Computation and Latency Analysis of SCR}
\label{app:scr_latency}

Let \(T_s\) denote the time required to encode the solve-static observation
history and goal once, \(T_d\) the candidate-dependent computation per
refinement iteration, and \(T_o\) the remaining solve-level overhead.
The original and SCR execution times are
\begin{align}
T_{\mathrm{orig}}(K)
&=
K(T_s+T_d)+T_o,
\label{eq:scr_original_latency}\\
T_{\mathrm{SCR}}(K)
&=
T_s+KT_d+T_o.
\label{eq:scr_cached_latency}
\end{align}
The resulting speedup is
\begin{equation}
S(K)
=
\frac{T_{\mathrm{orig}}(K)}
     {T_{\mathrm{SCR}}(K)}
=
\frac{K(T_s+T_d)+T_o}
     {T_s+KT_d+T_o},
\label{eq:scr_speedup_model}
\end{equation}
and the latency reduction is
\begin{equation}
R_T(K)
=
\frac{T_{\mathrm{orig}}(K)-T_{\mathrm{SCR}}(K)}
     {T_{\mathrm{orig}}(K)}
=
\frac{(K-1)T_s}
     {K(T_s+T_d)+T_o}.
\label{eq:scr_latency_reduction}
\end{equation}

These expressions distinguish the reduction in solve-static encoder calls
from the reduction in total planning latency.
Although SCR removes a fraction \(1-1/K\) of solve-static encoder calls,
candidate-dependent computation remains unchanged.
Consequently, its wall-clock benefit depends on the fraction of planning
time attributable to static encoding.

For \(T_s>0\), \(T_d>0\), and \(T_o\geq0\), the speedup in
Eq.~\ref{eq:scr_speedup_model} increases with \(K\), since
\begin{equation}
\frac{\partial S(K)}{\partial K}
=
\frac{
T_s(T_s+T_d+T_o)
}{
\left(KT_d+T_s+T_o\right)^2
}
>0.
\label{eq:scr_speedup_monotonicity}
\end{equation}
Moreover,
\begin{equation}
\lim_{K\rightarrow\infty}S(K)
=
1+\frac{T_s}{T_d}.
\label{eq:scr_asymptotic_speedup}
\end{equation}
Thus, SCR provides greater benefits as refinement depth increases, while
its asymptotic speedup remains bounded by the candidate-dependent
computation that cannot be reused.

\begin{figure*}[t]
\centering
\includegraphics[width=1.0\textwidth]{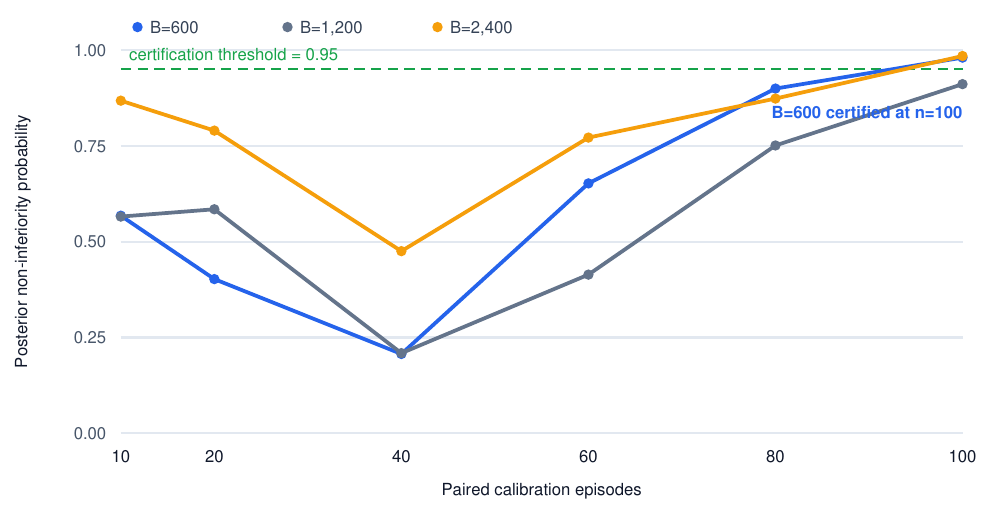}
\vspace{-0.9em}
\caption{
\textbf{Sequential certification trace of PSBC on PRISM/Cube.}
Curves show the posterior probability that each candidate budget is
non-inferior to Full planning within the prescribed tolerance as paired
calibration episodes accumulate. At 100 episodes, the 600-sequence
proposal crosses the 0.95 certification threshold and is returned as
the certified operating budget.
}
\label{fig:psbc_trace}
\end{figure*}

\section{Additional Experiments}
\label{sec:supplementary_experiments}

This section provides additional analyses of PSBC and SCR.
We first examine the variation in planning requirements across
model--task pairs, the stability of PSBC across evaluation seeds,
and its deployment efficiency.
We then compare alternative INTACT deployment strategies and analyze
the scaling behavior and numerical equivalence of SCR.

\subsection{Additional Analysis of PSBC}
\label{sec:supp_psbc}

\noindent\textbf {Model--task variation in planning budgets.} Fig.~\ref {fig:global_budget_vs_psbc} analyzes the mismatch between a universal planning budget and the model--task-specific budgets selected by PSBC. For each global budget, the figure counts the 16 evaluated pairs for which it is lower than, equal to, or higher than the corresponding PSBC-selected budget. No global setting provides a uniformly efficient allocation. For example, a 26.7\% budget falls below the selected budget for 6 pairs, whereas increasing it to 80.0\% avoids such underallocation but exceeds the selected budget for 14 of the 16 pairs. This exposes a fundamental trade-off of global budgeting: an aggressive setting may provide insufficient refinement for sensitive pairs, while a conservative setting spends unnecessary search on pairs whose performance saturates earlier. Together with the success results in Tab.~\ref {tab:planning_budget_success}, this analysis supports model--task-specific calibration as a more effective allocation strategy than imposing the same search budget across heterogeneous planners and tasks.

Fig.~\ref{fig:diff_backbone} further reports the available development-set
success rates across different planning budgets.
The success--budget profiles vary substantially across backbones and
tasks: some configurations reach strong performance with a small fraction
of Full planning, while others remain sensitive to additional refinement.
The diamonds and stars show the corresponding locked held-out results of
Full and PSBC-calibrated planning.
Development curves and held-out markers are shown together only to
visualize calibration behavior and transfer; they represent distinct
evaluation splits.

\noindent\textbf{Action refinement at equal aggregate success.}
Fig.~\ref{fig:action_vs_success} compares closed-loop success and
first-block action differences using the same INTACT Push-T checkpoint
and frozen development episodes.
Eight and 30 CEM iterations both achieve 96\% success on the frozen
100-episode development set. However, their first action blocks differ
in all 100 initial planning solves, with a mean per-dimension RMS
distance of 0.351 in native action units (paired bootstrap 95\% CI:
$[0.325,0.379]$). At the episode level, eight iterations rescue two
episodes and break two relative to 30 iterations. Thus, equal aggregate
success does not require identical actions or identical episode
outcomes. This comparison does not establish whether the optimizer
itself has converged.

\noindent\textbf{Sequential certification behavior.}
Fig.~\ref{fig:psbc_trace} presents the evolution of the posterior
non-inferiority probability as paired calibration episodes accumulate in
a PRISM/Cube calibration run.
At early evaluation points, the available evidence is insufficient to
certify the candidate budgets.
As additional paired episodes are observed, the estimated probabilities
are updated, and the 600-sequence candidate crosses the 0.95 certification
threshold at 100 episodes.
This example illustrates how PSBC retains uncertain candidates for further
evaluation and makes a decision only after sufficient paired evidence has
been collected.
Because posterior probabilities are updated from stochastic closed-loop
outcomes, the trajectories need not vary monotonically with the number of
episodes.

\noindent\textbf{Multi-seed evaluation.}
Fig.~\ref{fig:multiseed} compares Full and PSBC-calibrated planning over
evaluation seeds $\{0,1,42\}$ on Push-T and Cube.
PSBC generally retains success rates comparable to Full planning while
using substantially smaller deployment budgets.
The variation across seeds also remains small relative to the differences
among backbone--task pairs, supporting the stability of the reported
performance--budget trade-off.

\noindent\textbf{Deployment efficiency.}
Tab.~\ref{tab:deployment_efficiency} compares the planning time of Full
and PSBC-calibrated deployment on PRISM/Cube.
PSBC reduces the average planning time from 34.28 to 10.93 seconds per
episode, corresponding to a saving of 23.35 seconds per episode.
This result demonstrates that the reduced search budget translates
directly into lower deployment-time planning cost.

\noindent\textbf{Additional qualitative results.}
Fig.~\ref{fig:reacher_tworoom_qualitative} presents paired qualitative examples
on Reacher, TwoRoom, and Push-T.
For each task, LeWM+PSBC and Full LeWM are evaluated from the same initial
state toward the same goal, with the corresponding ground-truth trajectory
shown as reference.
In all three examples, LeWM+PSBC succeeds whereas Full LeWM fails despite
using a larger planning budget.
These cases provide qualitative evidence that increasing search does not
necessarily improve every closed-loop outcome and that a lower calibrated
budget can still produce a successful trajectory.
We emphasize that these examples illustrate individual paired episodes;
the aggregate performance comparison is reported in
Tab.~\ref{tab:planning_budget_success}.

\begin{figure*}[t]
\centering
\includegraphics[width=1.0\textwidth]{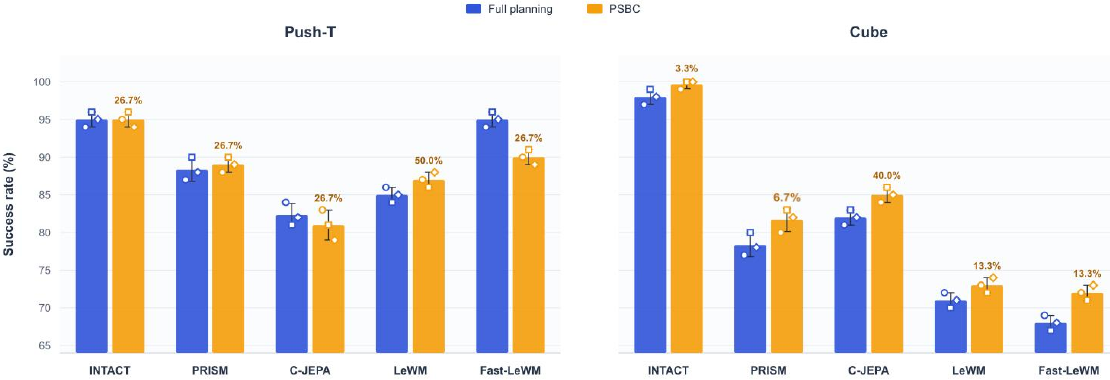}
\caption{
\textbf{Multi-seed success rates on Push-T and Cube.}
Bars and error bars report the mean and standard deviation over evaluation
seeds $\{0,1,42\}$, and markers show the individual results.
Percentages above the PSBC bars denote the corresponding deployment
budgets relative to Full planning.
}
\label{fig:multiseed}
\end{figure*}

\begin{figure*}[t]
\centering
\includegraphics[width=1.0\textwidth]{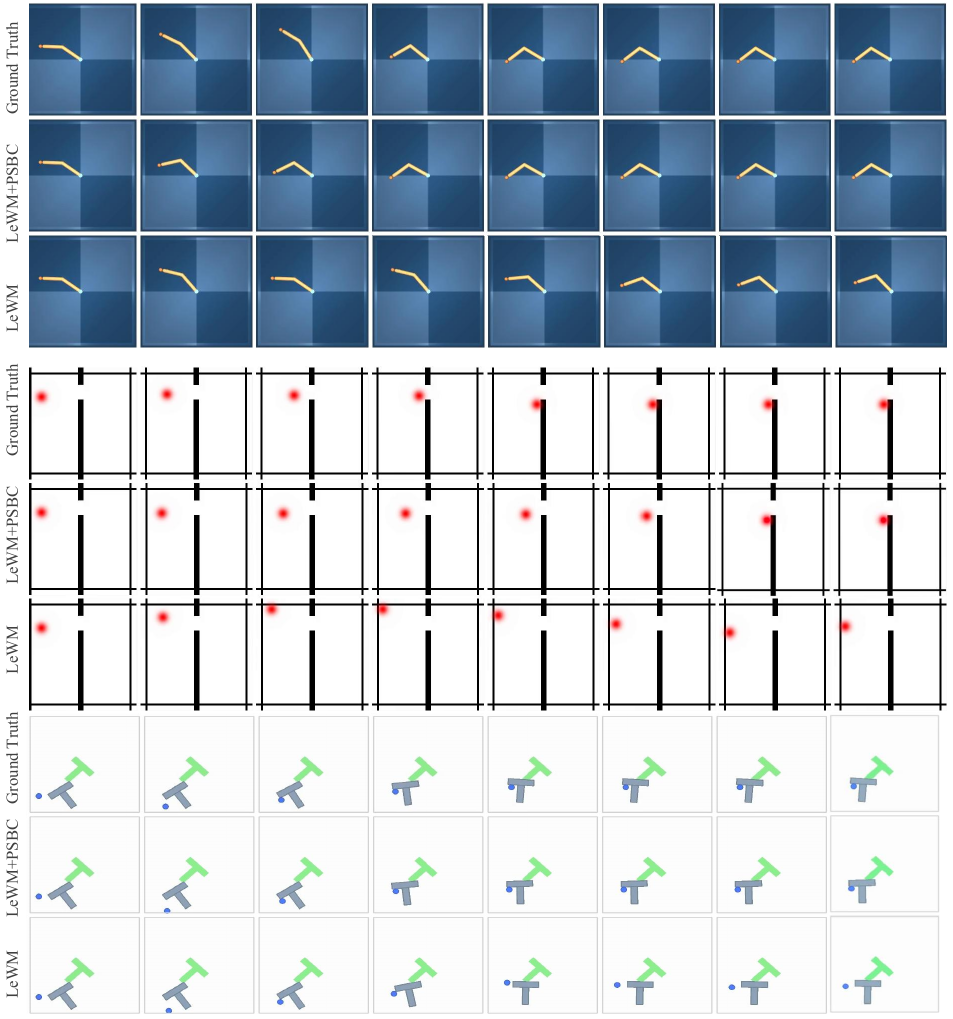}
\caption{
\textbf{Paired qualitative results on Reacher, TwoRoom, and Push-T.}
From identical initial states and goals, LeWM+PSBC succeeds while Full
LeWM fails; rows correspond to the three tasks and columns progress
from left to right.
}
\label{fig:reacher_tworoom_qualitative}
\end{figure*}

\subsection{INTACT Deployment Ablation}
\label{app:intact_deployment_ablation}

Tab.~\ref{tab:intact_deployment_comparison} compares Direct, Guarded-A,
Pure CEM, Full Actor-CEM, Actor-CEM+PSBC, and SufficientPlan on the frozen
Push-T development set.
Here, Direct executes the action predicted by the INTACT actor without
iterative search.
Guarded-A denotes the original INTACT guarded deployment strategy, which
uses the actor prediction together with a 384-sequence Actor-CEM
refinement budget.
Pure CEM performs Full-budget search without actor-guided initialization,
whereas Full Actor-CEM initializes the same Full-budget search using the
INTACT actor.
Actor-CEM+PSBC retains actor-guided initialization but uses the budget
selected by PSBC, and SufficientPlan further enables SCR within that
calibrated search.

Direct prediction provides the lowest search cost but achieves the lowest
success rate.
Guarded-A improves performance using a small refinement budget, while the
comparison between Pure CEM and Full Actor-CEM isolates the benefit of
actor-guided initialization.
PSBC retains this actor-guided benefit at a reduced search budget, and SCR
further removes repeated solve-static computation without changing the
selected budget or success outcome.

\subsection{Additional Analysis of SCR}
\label{sec:supp_scr}

\noindent\textbf{Scaling with planning depth.}
Tab.~\ref{tab:scr_budget_scaling} evaluates SCR on Fast-LeWM/Push-T over
budgets ranging from 300 to 9,000 scored sequences.
With one CEM iteration, no repeated solve-static encoding can be
eliminated across iterations.
Across repeated measurements, the uncached and SCR-enabled planners take
0.452 and 0.457 seconds per solve, respectively, corresponding to a
$0.99\times$ speedup with a paired bootstrap 95\% confidence interval of
$[0.876,1.125]$.
Thus, the one-iteration setting shows no reliable timing difference, as
expected when no cross-iteration reuse is available.

As the number of iterations increases, a larger fraction of repeated
observation and goal encoding can be reused.
At 30 iterations, SCR removes 96.7\% of solve-static encoder calls and
reduces measured planning latency from 11.4736 to 4.6796 seconds,
corresponding to a $2.45\times$ speedup.
This scaling behavior supports the mechanism described in
Fig.~\ref{fig:scr_mechanism}: the benefit of SCR increases when repeated
static encoding constitutes a larger fraction of the planning workload.

\noindent\textbf{Full-budget numerical equivalence.}
Aggregate success rates alone do not establish that SCR preserves the
internal behavior of the planner.
We therefore perform the numerical equivalence contract in
Tab.~\ref{tab:scr_full_contract} at the maximum 9,000-sequence budget
using common random numbers.
For both INTACT and Fast-LeWM, the uncached and SCR-enabled planners
produce identical candidate populations, candidate costs, CEM
distribution updates, final selected actions, and solver random-number
states.
The model buffers also remain unchanged, and both implementations score
the same number of candidate sequences.
These results verify that SCR changes the execution of solve-static
encoding but not the candidate-dependent planning computation.

\begin{table*}[t]
\centering
\caption{
\textbf{Sensitivity of the 600-sequence proposal on PRISM/Cube.}
Entries report the paired-episode inspection at which the proposal is
first certified.
``--'' indicates that the proposal is not certified within 100 paired
episodes.
}
\label{tab:psbc_sensitivity}
\small
\setlength{\tabcolsep}{9pt}
\renewcommand{\arraystretch}{1.10}
\begin{tabular}{cccc}
\toprule
& \multicolumn{3}{c}{$\boldsymbol{\tau}$}
\\
\cmidrule(lr){2-4}
$\boldsymbol{\delta}$ \textbf{(pp)}
& \textbf{0.90}
& \textbf{0.95}
& \textbf{0.99}
\\
\midrule
0 & 100 & --  & --  \\
2 & 100 & 100 & --  \\
5 & 80  & 80  & 100 \\
\bottomrule
\end{tabular}
\end{table*}

\section{Additional Analysis of PSBC}
\label{sec:additional_psbc_analysis}

We provide a mechanism-level analysis of PSBC on PRISM/Cube with
evaluation seed 1.
For this analysis, we use locked outcomes for Full planning and four
evaluated budgets, $\{300,600,1200,2400\}$, on the same frozen set of
100 paired development episodes.
These four budgets are used only for the diagnostic comparisons below
and are not presented as the complete PRISM proposal sequence.
Full planning uses 9,000 scored action sequences per planning solve.
The analyses in this section concern calibration behavior and are
separate from the episode-disjoint held-out results reported in the main
paper.

\subsection{Comparison with Simpler Selection Criteria}
\label{sec:psbc_selection_comparison}

We first compare PSBC with two simpler budget-selection criteria.
An \emph{empirical} rule accepts a candidate budget when its observed
success-rate degradation relative to Full planning does not exceed the
tolerance $\delta$.
A \emph{fixed-sample posterior} rule evaluates a candidate using all 100
paired episodes and accepts it when
\begin{equation}
q_{B,100}
=
\Pr\!\left(
    \Delta(B)\leq\delta
    \mid \mathcal{D}_{100}
\right)
\geq \tau,
\label{eq:fixed_sample_certification}
\end{equation}
where $\mathcal{D}_{100}$ denotes the complete paired calibration set.
PSBC uses the same paired posterior criterion but inspects the evidence
sequentially according to its predefined stopping rule.

Tab.~\ref{tab:psbc_selection_comparison} reports the results under the
default tolerance $\delta=2$ percentage points and posterior threshold
$\tau=0.95$.
The 300-sequence candidate achieves 83\% success, compared with 85\% for
Full planning.
Its empirical difference of $-2$ percentage points lies exactly at the
prescribed tolerance boundary and would therefore be accepted by the
empirical rule.
However, its paired outcomes contain 10 rescued and 12 broken episodes,
yielding a posterior non-inferiority probability of only 0.503.
Thus, satisfying the empirical tolerance alone does not provide
sufficient posterior evidence for certification.

The 1,200-sequence candidate provides a less extreme example.
Although its empirical success rate is 4 percentage points higher than
that of Full planning, its posterior non-inferiority probability of
0.911 does not reach the prescribed 0.95 threshold.
By contrast, the 600- and 2,400-sequence candidates obtain posterior
probabilities of 0.981 and 0.984, respectively.
The non-monotonic posterior values across budgets are consistent with
our formulation, which does not assume monotonic closed-loop performance
as the planning budget increases.

\subsection{Sequential Inspection Efficiency}
\label{sec:psbc_sequential_efficiency}

We next analyze the evaluation count induced by the predefined
sequential stopping rule.
Applying this rule retrospectively to the locked paired outcomes rejects
the 300-sequence proposal after 20 paired episodes and subsequently
certifies the 600-sequence proposal after 100 paired episodes:
\begin{equation}
300 @ 20
\;\longrightarrow\;
\textsc{Insufficient},
\qquad
600 @ 100
\;\longrightarrow\;
\textsc{Sufficient}.
\label{eq:psbc_replayed_trajectory}
\end{equation}
PSBC stops after certifying the 600-sequence proposal; therefore, larger
budgets are not visited under this replayed trajectory.

We count one candidate-budget--episode evaluation whenever a
reduced-budget planner is evaluated on one episode.
The 100-episode Full-reference vector is computed once and shared across
all paired comparisons, so it is excluded from the candidate-side counts
reported below.
Under the sequential policy, the candidate-side evaluation count is
\begin{equation}
C_{\mathrm{PSBC}}
=
20+100
=
120.
\label{eq:psbc_candidate_side_cost}
\end{equation}
An ordered fixed-$n$ posterior baseline that evaluates the same two
visited proposals on all 100 episodes requires
\begin{equation}
C_{\mathrm{fixed\text{-}n}}
=
2\times100
=
200.
\label{eq:fixed_n_candidate_side_cost}
\end{equation}
The resulting candidate-side reduction is therefore
\begin{equation}
1-
\frac{C_{\mathrm{PSBC}}}
     {C_{\mathrm{fixed\text{-}n}}}
=
1-\frac{120}{200}
=
40.0\%.
\label{eq:psbc_candidate_side_reduction}
\end{equation}

This comparison shows that sequential rejection avoids evaluating the
300-sequence proposal on the remaining 80 episodes before proceeding to
the certified 600-sequence budget.
The reported saving is an evaluation-count reduction implied by applying
the predefined stopping policy to locked paired outcomes, rather than a
separate wall-clock measurement.
It also concerns offline candidate-side calibration and is distinct from
the number of action sequences scored during held-out deployment.

\subsection{Retrospective Analysis of Early Certificates}
\label{sec:retrospective_certification}

\noindent\textbf{Retrospective evidence after an early decision.}
Tab.~\ref{tab:early_certificate_audit} applies the stated ascending
proposal and sequential inspection rules to four locked PRISM
calibration traces. Exact integration of the paired Dirichlet posterior
confirms the first certified proposal and stopping look in each case.
The two Push-T proposals are certified at $n=20$, but their posterior
probabilities fall below the certification threshold when the unused
outcomes from the same 100-episode set are inspected retrospectively.
The two Cube proposals remain above the threshold. This analysis does
not change the information available to the original stopping rule;
it shows that a posterior certificate at one inspection time need not
remain above the same threshold after additional observations.

\begin{table*}[t]
\centering
\caption{
\textbf{SCR scaling on Fast-LeWM/Push-T.}
SCR provides no reliable speedup at one iteration, while its latency
benefit increases with search depth.
}
\label{tab:scr_budget_scaling}
\small
\setlength{\tabcolsep}{7pt}
\renewcommand{\arraystretch}{1.12}

\resizebox{\textwidth}{!}{
\begin{tabular}{ccccccc}
\toprule
\textbf{Budget}
& \textbf{Iter.}
& \textbf{Uncached (s)}
& \textbf{SCR (s)}
& \textbf{Static Enc. Reduction}
& \textbf{Latency Reduction}
& \textbf{Speedup} \\
\midrule

300
& 1
& 0.4521
& 0.4571
& 0.0\%
& $-1.1$\%
& 0.99$\times$ \\

600
& 2
& 1.2422
& 0.8000
& 50.0\%
& 35.6\%
& 1.55$\times$ \\

1,200
& 4
& 2.0664
& 1.0660
& 75.0\%
& 48.4\%
& 1.94$\times$ \\

2,400
& 8
& 3.4617
& 1.7080
& 87.5\%
& 50.7\%
& 2.03$\times$ \\

4,500
& 15
& 6.2457
& 2.6330
& 93.3\%
& 57.8\%
& 2.37$\times$ \\

9,000
& 30
& 11.4736
& 4.6796
& 96.7\%
& 59.2\%
& 2.45$\times$ \\

\bottomrule
\end{tabular}
}
\end{table*}

\subsection{Sensitivity of the Certified Proposal}
\label{sec:psbc_sensitivity}

We finally examine how the evidence required to certify the
600-sequence proposal changes with the non-inferiority tolerance
$\delta$ and posterior threshold $\tau$.
We consider
\begin{equation}
\delta\in\{0,2,5\}\ \text{percentage points},
\qquad
\tau\in\{0.90,0.95,0.99\},
\label{eq:psbc_sensitivity_grid}
\end{equation}
using the predefined inspection times
$\{10,20,40,60,80,100\}$.

Tab.~\ref{tab:psbc_sensitivity} reports the first inspection time at
which the 600-sequence proposal reaches the corresponding certification
threshold.
Under the default setting $(\delta=2,\tau=0.95)$, the proposal is
certified after 100 paired episodes.
With the same tolerance and a lower threshold of 0.90, certification
also occurs at 100 episodes, whereas the more conservative threshold of
0.99 is not reached within the available evidence.
A wider 5-point tolerance permits certification after 80 episodes for
$\tau\in\{0.90,0.95\}$ and after 100 episodes for $\tau=0.99$.
Under the strict zero-point tolerance, the proposal is certified only
for $\tau=0.90$.

These results show that a wider tolerance can reduce the evidence
required for certification, while a more conservative posterior
threshold can delay or prevent certification within the available
calibration horizon.
For settings in which the 600-sequence proposal is not certified, we do
not infer the final deployment budget from this analysis, because doing
so would require evaluating the remaining budgets in the complete PRISM
proposal sequence.

\paragraph{Scope.}
This section is a mechanism-level case study rather than a cross-task
claim about calibration savings.
It shows that paired posterior certification can reject a proposal that
would be accepted using only an empirical tolerance, that sequential
inspection can reduce candidate-side calibration evaluations, and that
the required evidence varies predictably with $\delta$ and $\tau$.
Cross-backbone and cross-task held-out deployment results are reported
separately in Tab.~\ref{tab:planning_budget_success}.

\begin{table*}[t]
\centering
\caption{
\textbf{Full-budget numerical equivalence contract for SCR.}
Uncached and SCR-enabled planners are evaluated with common random
numbers at the maximum 9,000-sequence budget.
The contract verifies the internal planning trajectory rather than
only aggregate task success.
}
\label{tab:scr_full_contract}
\small
\setlength{\tabcolsep}{7pt}
\renewcommand{\arraystretch}{1.10}
\begin{tabular}{lcc}
\toprule
\textbf{Contract}
& \textbf{INTACT}
& \textbf{Fast-LeWM} \\
\midrule
Candidate population/hash equal
& \cmark & \cmark \\
Candidate cost equal
& \cmark & \cmark \\
CEM mean/std updates equal
& \cmark & \cmark \\
Final selected action equal
& \cmark & \cmark \\
Solver RNG state equal
& \cmark & \cmark \\
Model buffers unchanged
& \cmark & \cmark \\
Scored sequences equal
& 9,000/9,000 & 9,000/9,000 \\
\midrule
\textbf{Full-9000 contract}
& \textbf{PASS} & \textbf{PASS} \\
\bottomrule
\end{tabular}
\end{table*}

\section {Limitations}
\label {sec:limitations} 
\vspace {-0.3em}  
SufficientPlan has two main limitations. First, PSBC calibrates the planning budget on a frozen development distribution. Although the selected budget is evaluated without further tuning on episode-disjoint held-out data, substantial changes in the model, task dynamics, or goal distribution may require recalibration. Second, SCR applies only to computations that remain invariant within a planning solve, and its practical speedup depends on the fraction of runtime attributable to such solve-static computation. When candidate-dependent rollout and cost evaluation dominate the planning workload, the achievable latency reduction may be limited.


\end{document}